\documentclass{tlp}

\submitted{24 May 2007}
\revised{22 March 2008, 19 September 2008}
\accepted{27 November 2008}

\usepackage{amssymb}
\usepackage{amsfonts}
\usepackage{epsfig}
\usepackage{amsmath}
\usepackage{latexsym}
\usepackage{verbatim}
\usepackage{url}

\def\ni{\noindent}
\def\naf{\; not \; }

\long\def\COMMENT#1\ENDCOMMENT{\message{(Commented text...)}\par}

\newtheorem{observation}{Observation}
\newtheorem{theorem}{Theorem} 
\newtheorem{definition}{Definition} 
\newtheorem{proposition}{Proposition}
\newtheorem{example}{Example}
\newtheorem{lemma}{Lemma}[section]
\newtheorem{corollary}{Corollary} 
 
\newtheorem{remark}{Remark}

\newcommand{\uffa}{\mbox{$\: {\tt : \!\! - }\:$}} 
\newcommand{\qed}{\hfill $\Box$}
\newcommand{\ASP}{$\mathbb{ASP-PROLOG}$}

\title
[Justifications for Logic Programs under Answer Set Semantics]
{\emph{Justifications} for Logic Programs under Answer Set Semantics} 
\author
[Enrico Pontelli, Tran Cao Son, and Omar Elkhatib]
{
Enrico Pontelli, Tran Cao Son, and Omar Elkhatib \\
Department of Computer Science\\
New Mexico State University \\
{\tt epontell|tson|okhatib@cs.nmsu.edu}\\
}
\begin{document}

\maketitle

\begin{abstract}
The paper introduces the notion of \emph{off-line justification}
for Answer Set Programming (ASP). Justifications provide
a graph-based explanation of the truth value of an atom w.r.t.
a given answer set. 
The paper extends also
this notion to provide justification of atoms \emph{during} the
computation of an answer set (\emph{on-line justification}), and presents
an integration of on-line justifications within the computation model
of {\sc Smodels}.
Off-line and on-line justifications provide  useful tools to enhance
understanding of ASP, and they offer a basic data
structure to support  methodologies and tools for
\emph{debugging} answer set programs. A preliminary implementation
has been developed in \ASP.
\end{abstract}


\section{Introduction}

\emph{Answer set programming (ASP)} is a programming paradigm 
\cite{smodels-constraint,mar99,lif02a} 
based on logic programming 
under answer set semantics \cite{gel88}. 
ASP is a  {\em highly declarative} paradigm. In order 
to solve a problem $P$, we specify it as a logic program $\pi(P)$\/,
whose answer sets correspond one-to-one to solutions of $P$, and 
can be computed using an answer set solver.  ASP is also attractive because of
its numerous building block results 
(see, e.g., \cite{Baral03}). This can be seen in the following example. 

\begin{example}
Consider the problem of computing the Hamiltonian cycles of a graph. The graph
can be encoded as a collection of facts, e.g.,
\[{\tt \begin{array}{lclclcl}
\texttt{vertex(a).} & \hspace{.4cm} & \texttt{vertex(b).} & \hspace{.4cm} & 
	\texttt{vertex(c).} & \hspace{.4cm} & \texttt{vertex(d).}\\
\texttt{edge(a,b).} & & \texttt{edge(a,c).} & & \texttt{edge(b,d).} & & \texttt{edge(b,c).}\\
\texttt{edge(c,d).} && \texttt{edge(d,a).} 
  \end{array}}
\]
A program contains rules, in the form of Horn clauses; in our case:
\[{\tt \begin{array}{lcl}
\multicolumn{3}{l}{\%\% \:\:\textit{Select an edge}}\\
\texttt{in(U,V)} & \uffa & \texttt{edge(U,V)}, not\:\texttt{nin(U,V)}.\\
\texttt{nin(U,V)} & \uffa & \texttt{edge(U,V)}, not \: \texttt{in(U,V)}.\\
\multicolumn{3}{l}{\%\% \:\:\textit{Traverse each node only once}}\\
\texttt{false} & \uffa & \texttt{vertex(U), vertex(V), vertex(W)}, \\
               &       & \texttt{V} \neq \texttt{W, in(U,V), in(U,W)}.\\
\texttt{false} & \uffa & \texttt{vertex(U), vertex(V), vertex(W)}, \\
               &       & \texttt{U} \neq \texttt{V, in(U,W), in(V,W)}.\\
\multicolumn{3}{l}{\%\% \:\:\textit{Reachability of nodes}}\\
\texttt{reachable(U)} & \uffa & \texttt{vertex(U), in(a,U).}\\
\texttt{reachable(V)} & \uffa & \texttt{vertex(V), vertex(U), reachable(U), in(U,V).}\\
\multicolumn{3}{l}{\%\%\:\:\textit{Each vertex reachable from a}}\\
\texttt{false} & \uffa & \texttt{vertex(U), U} \neq \texttt{a}, not\:\texttt{reachable(U).}
  \end{array}}
\]
It can be shown that every answer set of the program consisting of the rules representing 
the graph and the above rules corresponds to an Hamiltonian cycle of the graph
and vice versa. Furthermore, the program has no answer set if and only if the graph 
does not have an  Hamiltonian cycle. 
\hfill $\Box$
\end{example}
The popularity of ASP has grown significantly over the years, finding innovative and 
highly declarative applications
in a variety of domains, such as intelligent agents~\cite{Baral03,BalducciniGN06}, planning~\cite{lif99d},
software modeling and verification~\cite{HeljankoN03}, complex systems diagnosis~\cite{BalducciniG03},
and phylogenetic inference~\cite{ErdemLR06}. 

The growing popularity of ASP, especially in domains like non-monotonic
and commonsense reasoning, has been supported by the development of excellent 
inference engines~\cite{AngerGLNS05,eiter98a,GebserKNS07,GiunchigliaLM04,lin02a,sim02}. 
On the other hand,
a source of difficulties in learning to use ASP
lies in the lack of \emph{methodologies} and \emph{tools}
which can assist users in understanding
a program's 
behavior and debugging it. The highly  declarative nature of the ASP framework and
the ``hand-off'' execution style of ASP  leave a programmer 
with little information that helps in explaining the behavior of the programs, except
for the program itself. For example, the additional information that can 
be gained by exploring the intermediate state of an execution (e.g., value
of variables) of an imperative program using a debugger does not have any
equivalent in the context of ASP. 
This situation is especially difficult when
the program execution produces unexpected outcomes, 
e.g., incorrect  or missing answer sets. 
In this sense, this paper shares the spirit of 
other attempts in developing tools and 
methodologies for understanding and debugging of ASP programs,\footnote{
  Abusing the notation, we often refer to a logic program under 
  the answer set semantics as an ``ASP program'' whenever it is 
  clear from the context what it refers to.
}
as in  
\cite{BrainGP+a07,BrainGP+b07,ElKhatibPS05,dlvdeb}.

Although the traditional language of logic programming under answer set semantics,
e.g., referred to as AnsProlog in \cite{Baral03} or A-Prolog \cite{GelfondL02},
is syntactically close to Prolog, the execution model
and the semantics are sufficiently different to make 
debugging techniques developed for Prolog impractical.
For example, the traditional \emph{trace-based}
debuggers~\cite{RoychoudhuryRR00} 
(e.g., Prolog four-port debuggers), used to
trace the entire proof search tree (paired with execution control mechanisms, like
spy points and step execution), are cumbersome
in  ASP, since: 
\begin{itemize}
\item Trace-based debuggers provide the entire search sequence, including 
	the failed paths, which might be irrelevant in understanding 
	how specific elements are introduced in an answer set. 
\item The process of computing answer sets is bottom-up, and 
	the determination of the truth value of one atom is intermixed with
	the computation of other atoms; a direct tracing makes it  hard
	to focus on what is relevant to one particular atom. 
	This is illustrated in the following example. 
\begin{example}
\label{exbelow}
Consider the following simple program.
\[{\tt \begin{array}{lclclcl}
\texttt{s} &\uffa & \texttt{r}. & \hspace{1cm} & \texttt{s} & \uffa & \texttt{t}.\\
\texttt{r} &\uffa & \texttt{a}. & & \texttt{t.} 
       \end{array}}
\]
The program $P$ has a unique answer set, 
$M=\{\texttt{s}, \texttt{t}\}$. In particular, $\texttt{t} \in M$, since $\texttt{t}$ 
appears as a fact in the program, and $\texttt{s} \in M$ because of the 
rule \mbox{\texttt{s \uffa t}} and $\texttt{t} \in M$. 
In this process, there is no need to expose the processing of
the rule 
\mbox{\texttt{s \uffa r}} to the user, since \texttt{r $\not \in M$}.\hfill$\Box$
\end{example}

\item Tracing repeats previously performed executions, 
	degrading debugging  performance and confusing the programmer.
\end{itemize}
We address these issues by elaborating the concept of
{\em off-line justification} for ASP. This notion
is an evolution of the concept of \emph{justification}, 
proposed to justify truth values in tabled Prolog~\cite{RoychoudhuryRR00,PemmasaniGDRR04}.
Intuitively, an off-line justification of an atom 
w.r.t. an answer set is a graph encoding the reasons for the atom's
truth value. This notion can be used to explain the presence 
or absence of an atom in an answer set, and provides the basis
for building  a {\em justifier} for answer set solvers.
In this paper, we develop this concept and investigate
its properties.

The notion of off-line justification is helpful when investigating the
content of  one (or more) answer sets.
When the program  does not have answer sets, a different 
type of justification is needed. This leads us to the notion 
of {\em on-line justification}, which provides justifications with 
respect to a \emph{partial} and/or (sometimes) \emph{inconsistent} 
interpretation. An on-line justification is \emph{dynamic}, 
in that it can be obtained at any step of the answer set 
computation, provided that the computation process follows
certain strategies. The intuition is to allow the programmer to interrupt
the computation (e.g., at the occurrence of certain events, such as
assignment of a truth value to a given atom)
and to use the on-line justification to explore the motivations behind 
the content of the
partial interpretation (e.g., why a given atom is receiving
conflicting truth values).
We describe a \emph{generic} model of on-line justification and
a version specialized to the execution model of 
 {\sc Smodels}~\cite{sim02}. The latter has been implemented
in \ASP\ \cite{asp-prolog}.

Justifications are offered as first-class citizens of a Prolog
system, enabling the programmer to use Prolog programs to reason
about ASP computations. Debugging is one of the possible uses of 
the notion of justification  developed in this paper.

\section{Background}

In this paper, we  focus 
on a logic programming
language with negation as failure---e.g., the language
of {\sc Smodels} without weight constraints and choice rules
\cite{sim02}. 

\subsection{Logic Programming Language}

Each program $P$ is associated with a signature 
$\Sigma_P=\langle {\cal F}, \Pi, {\cal V} \rangle$, where
\begin{itemize}
\item $\cal F$ is a finite set of \emph{constants},
\item $\cal V$ is a set of \emph{variables}, and
\item  $\Pi$ is a finite set of \emph{predicate} symbols.
\end{itemize}
In particular, we assume that
$\top$ (stands for $true$) and 
$\bot$ (stands for $false$) are zero-ary predicates in  $\Pi$. 
A \emph{term} is 
a constant of $\cal F$ or a variable of $\cal V$. 
An \emph{atom}  is of the 
form $p(t_1,\dots,t_n)$, where $p\in \Pi$, and $t_1,\dots,t_n$ are terms. In 
particular, a term (atom) is said to be \emph{ground} if there are no
occurrences of elements of $\cal V$ in it. 

A \emph{literal} is either  an atom (\emph{Positive Literal}) or 
a formula of the form $not\:a$, where $a$ is an atom (\emph{NAF Literal}). In what follows, 
we will identify with $\cal L$ the set of all ground literals. 

A \emph{rule}, $r$, is of the form 
\begin{equation} \label{rule} 
  h \:\uffa\: b_1, \dots, b_n.
\end{equation}
($n\geq 0$) where $h$ is an atom and $\{b_1, \dots, b_n\}\subseteq {\cal L}$. The atom
$h$ is referred to as the \emph{head} of the rule, while the set
of literals $\{b_1,\dots, b_n\}$ represents the \emph{body} of the rule. 
Given a rule $r$, we denote
$h$ with $head(r)$ and we use $body(r)$ to denote the set
 $\{b_1,\dots,b_n\}$. We also denote with
$pos(r)  = body(r) \cap {\cal A}$---i.e., all the elements of the body
that are not negated---and with 
$neg(r)  =  \{ a \:|\: (not\: a) \in body(r)\}$---i.e., the atoms that appear
negated in the body of the rule. 

Given a rule $r$, we denote with
$ground(r)$ the set of all rules obtained by consistently replacing the
variables in $r$ with constants from $\cal F$ (i.e., the \emph{ground
instances} of $r$).

We identify special types of rules:
\begin{itemize}
\item A rule $r$ is \emph{definite} if $neg(r) = \emptyset$\/;
\item A rule $r$ is a \emph{fact} if $neg(r) \cup pos(r) = \emptyset$\/; for the 
	sake of readability, the fact
\[ h \uffa . \]
	will be simply written as
\[ h .\]
\end{itemize}

A program $P$ is a set of rules. A program with variables is understood
as a shorthand for the set of all ground instances of the rules in $P$; we will use the 
notation:
\[ ground(P) = \bigcup_{r\in P} ground(r) \]
A program is \emph{definite} if it contains only definite rules.

The answer set semantics of a program (Subsection \ref{semantics}) 
is highly dependent on the truth value of atoms occurring in the 
negative literals of 
the program. For later use, we denote with 
$NANT(P)$ the atoms which appear in NAF literals in $P$---i.e.,
\[
NANT(P) = \{a \mid a \:\textnormal{is a ground atom },\:  \exists r \in ground(P):\: a\in neg(r)\}.
\]
We will also use ${\cal A}_P$ to denote the Herbrand base of 
a program $P$. For brevity, we will often write $\cal A$ instead 
of ${\cal A}_P$. 

\begin{example}\label{exa}
Let us consider the program $P_1$ containing the rules:
\[{\tt
\begin{array}{clclcclcl}
(r_1) & \texttt{q} & \uffa & \texttt{a}, not\:\texttt{p}. & \hspace{1cm} & 
(r_2) & \texttt{p} & \uffa & \texttt{a}, not\:\texttt{q}.\\
(r_3) & \texttt{a} & \uffa & \texttt{b}. & & (r_4) & \texttt{b}.
\end{array}}
\]
The rule $r_3$ is definite, while the rule $r_4$ is a fact. For the rule
$r_1$ we have:
\begin{itemize}
\item $head(r_1)=\texttt{q}$ 
\item $body(r_1)=\{\texttt{a}, not\:\texttt{p}\}$
\item $pos(r_1) = \{\texttt{a}\}$
\item $neg(r_1) = \{\texttt{p}\}$
\end{itemize}
For $P_1$, we have  $NANT(P_1) = \{\texttt{p}, \texttt{q}\}$.
\hfill $\Box$
\end{example}

\subsection{Answer Set Semantics and Well-Founded Semantics}
\label{semantics}
 
We will now review two important semantics of logic programs,
the answer set semantics and the well-founded semantics. The former 
is foundational to ASP and the latter is important for the development of 
our notion of a justification. We will also briefly discuss the basic 
components of ASP systems.  

\subsubsection{Interpretations and Models}
A \emph{(three-valued) interpretation} $I$ is a pair $\langle I^+,I^- \rangle$,
where $I^+ \cup I^- \subseteq {\cal A}$ and $I^+ \cap I^- = \emptyset$. Intuitively,
$I^+$ collects the knowledge of the atoms that are known to be true, while
$I^-$ collects the knowledge of the atoms that are known to be false. 
$I$ is a \emph{complete interpretation} if $I^+ \cup I^- = {\cal A}$. If 
$I$ is not complete, then it means that there are atoms whose truth value
is \emph{undefined} with respect to $I$. For convenience, 
we will often say that an atom $a$ is undefined in $I$ 
and mean that the truth value of $a$ is undefined in $I$. 

Let $P$ be a program and $I$ be an interpretation. 
A positive literal $a$ is satisfied by $I$, denoted by $I \models a$, if 
$a \in I^+$. 
A NAF literal $not \; a$ is satisfied by $I$---denoted by
$I \models not\:a$---if $a \in I^-$. A set of literals $S$ 
is satisfied by $I$ ($I\models S$) if $I$ satisfies each literal in $S$. The notion of
satisfaction is easily extended to rules and programs as follows.
A rule $r$ is satisfied by $I$ if $I\not\models body(r)$ or
 $I\models head(r)$. $I$ is a \emph{model} of a program if it satisfies all its rules.
An atom $a$ is \emph{supported} by $I$ in $P$ if there exists 
$r \in P$ such that $head(r) = a$ and $I \models body(r)$.

We introduce two partial orders on the set of interpretations:
\begin{itemize}
\item For two interpretations $I$ and $J$, we say that
 $I \sqsubseteq J$ iff $I^+ \subseteq J^+$ and $I^- \subseteq J^-$
\item For two interpretations $I$ and $J$, we say that
 $I \preceq J$ iff $I^+ \subseteq J^+$
\end{itemize}
We will denote with $\cal I$ the set of all possible interpretations and
with $\cal C$ the set of complete interpretations.
An important property~\cite{llo87} of definite programs is that for each program $P$
there exists a unique model $M_P$ which is $\preceq$-minimal
over $\cal C$. $M_P$ is called the \emph{least Herbrand model} of $P$.

\subsubsection{Answer Set Semantics}

For an interpretation 
$I$ and a program $P$, the \emph{reduct} 
of $P$ w.r.t. $I$ (denoted by $P^I$) is the program
obtained from $P$ by deleting
{\em (i)} each rule $r$ such that $neg(r)\cap I^+ \neq \emptyset$, and
{\em (ii)} all NAF literals in the bodies of the remaining rules. Formally,
\[P^I = \left\{
	head(r) \uffa pos(r) \:\:|\:\:
		 r \in P, \:\: neg(r) \cap I^+ = \emptyset
	\right\}
\]
Given a complete interpretation $I$, 
observe that the program $P^I$ is a definite program. 
A complete interpretation $I$ is an \emph{answer set} \cite{gel88}
of $P$ if $I^+$ is the least Herbrand model of $P^I$~\cite{apt94a}.

\begin{example}briefly
Let us reconsider the program $P_1$ in Example~\ref{exa}. If we consider
the interpretation $I = \langle \{\texttt{b,a,q}\},\{\texttt{p}\}\rangle$, then the 
reduct $P_1^I$ will contain the rules:
\[{\tt\begin{array}{lclclcl}
\texttt{q} & \uffa & \texttt{a}. & \hspace{1cm} & \texttt{a} & \uffa & \texttt{b}.\\
\texttt{b}.
      \end{array}
}
\]
It is easy to see that $\{\texttt{a},\texttt{b},\texttt{q}\}$ is the least Herbrand model of this program; thus,
$I$ is an answer set of $P_1$. 
\qed
\end{example}

For a definite program $P$ and an interpretation $I$, 
the immediate consequence operator $T_P$ is defined by 
\[
T_P(I) = \{a \mid \exists r \in P, head(r) = a, I \models body(r)\}. 
\]
$T_P$ is monotone and has a least fixpoint \cite{VanEmdenK76}. 
The fixpoint of $T_P$ will be denoted by $lfp(T_P)$. 

\subsubsection{Well-Founded Semantics}
Let us describe the \emph{well-founded semantics}, following  the 
 definition proposed in \cite{apt94a}. We note that this definition 
is slightly different from the original definition of the well-founded 
semantics in \cite{VanGelderRS91}. 
Let us start by recalling some auxiliary definitions.

\begin{definition} \label{tpj}
Let $P$ be a program, $S$ and $V$ be sets of atoms from $\cal A$.
The set $T_{P,V}(S)$ ({\em immediate consequence of S w.r.t P
and V}) is defined as follows:
\[
T_{P,V}(S) =  \{ a \mid
\exists r \in P, head(r) = a, pos(r) \subseteq S, neg(r) \cap V = \emptyset
\}
\]
\end{definition}

It is easy to see that, if $V$ is fixed,  the operator 
is monotone with respect to $S$. Against, we use $lfp(.)$ to 
denote the least fixpoint of this operator when $V$ is fixed. 

\begin{definition} \label{kui}
Let $P$ be a program and $P^+$ be the set of definite rules in $P$.
The sequence $(K_i,U_i)_{i\ge 0}$ is defined as follows:
\[
\begin{array}{lclclcl}
K_0 & = & lfp(T_{P^+}) & \hspace{1cm} & 
U_0 & = & lfp(T_{P,K_0}) \\
K_i & = & lfp(T_{P,U_{i-1}}) &&
U_i & = & lfp(T_{P,K_{i}}) \\
\end{array}
\]
\end{definition}
Let $j$ be the first index of the 
computation such that $\langle K_j,U_j \rangle = \langle K_{j+1}, U_{j+1} \rangle$.
We will denote with $WF_P = \langle W^+,W^- \rangle$ the 
(unique) \emph{well-founded} model of $P$,
where $W^+ = K_j$ and $W^- = {\cal A} \setminus U_j$. 

briefly
\begin{example}
Let us reconsider the program $P_1$ of Example~\ref{exa}. The computation of the well-founded
model proceeds as follows:
\[\begin{array}{lclcl}
K_0 & = & \{\texttt{b},\texttt{a}\}\\
U_0 & = & \{\texttt{a},\texttt{b},\texttt{p},\texttt{q}\}\\
K_1 & = & \{\texttt{a},\texttt{b}\}&  =& K_0\\
U_1 & = & \{\texttt{a},\texttt{b},\texttt{p},\texttt{q}\}&  =& U_0
  \end{array}
\]
Thus, the well-founded model will be $\langle \{\texttt{a},\texttt{b}\}, \emptyset\rangle$. Observe that
both $\texttt{p}$ and $\texttt{q}$ are undefined in the well-founded model.\qed
\end{example}

\subsection{Answer Set Programming Systems}

Several efficient ASP solvers have been developed, such as 
{\sc Smodels}~\cite{NiemelaS97}, {\sc DLV}~\cite{eiter98a},
{\sc Cmodels}~\cite{GiunchigliaLM04}, {\sc ASSAT}~\cite{lin02a},
and {\sc CLASP}~\cite{GebserKNS07}.
One of the most popular ASP solvers is {\sc Smodels}
\cite{NiemelaS97,sim02} which comes with {\sc Lparse},
a grounder. {\sc Lparse} takes as input a logic
program $P$ and produces as output a simplified version of 
$ground(P)$. The output of {\sc Lparse} is in turn accepted by 
{\sc Smodels}, and used to produce the answer sets of 
$P$ (see  Figure~\ref{lparse_smodel}).

\begin{figure}[htb]
\centerline{\psfig{figure=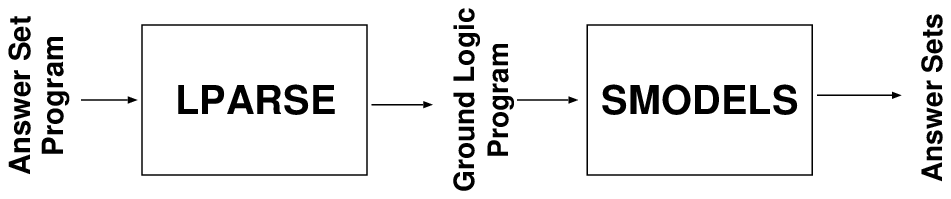,width=.6\textwidth}}
\caption{The {\sc Lparse/{\sc Smodels}} System}
\label{lparse_smodel}
\end{figure}
The {\sc Lparse}/{\sc Smodels} system supports
several extended  types 
of literals, such as the \emph{cardinality literals}, which 
are of the form: $L\ \{l_1, \ldots, l_n\}\ U$,
where $L$ and $U$ are integers, $L \le U$, and 
$l_1, \ldots, l_n$ are literals.
The cardinality literal is satisfied by an answer set $M$ if the number 
$x$ of literals in $\{l_1,\dots,l_n\}$ that are true in $M$ is such
that $L \leq x \leq U$.

The back-end engine, {\sc Smodels} in Figure~\ref{lparse_smodel}, produces the collection
of answer sets of the input program. Various control options can be provided
to guide the computation---e.g., establish a limit on the number of
answer sets provided or request the answer set to contains specific atoms.

We note that all of the available ASP solvers~\cite{AngerGLNS05,eiter98a,GebserKNS07,GiunchigliaLM04,lin02a}
operate in a similar fashion as {\sc Smodels}. {\sc DLV} uses its own grounder while 
others use {\sc Lparse}. New grounder programs have also been
recently proposed, e.g., Gringo in \cite{GebserTT07}. 
SAT-based answer set solvers rely on SAT-solver in computing answer sets
\cite{GiunchigliaLM04,lin02a}. 

\section{Explanations}

The traditional methodology employed in ASP relies on
encoding each problem $Q$ 
as a logic program $\pi(Q)$, whose answer sets are in one-to-one
correspondence with the  solutions 
of $Q$. From the software development perspective, it would 
be important to address the question 
\emph{``why is $M$ an answer set of the program 
$P$?''} This question gives 
rise to the question  ``why does an atom $a$ belong to $M^+$ 
(or $M^-$)?'' Answering this question 
can be very important, in that it provides us with explanations 
regarding the presence (or absence) of different  atoms in $M$. 
Intuitively, we view answering these questions as the 
``declarative'' parallel of answering questions of the type
``why is $3.1415$ the value of the  variable $x$?'' in the context
of imperative languages---a question that can be
typically answered by
producing and analyzing an \emph{execution trace} 
(or \emph{event trace}~\cite{Auguston00}).

The objective of this section is to develop the notion of
\emph{explanation}, as  a graph structure used to describe the
``reason'' for the truth value of an atom w.r.t. a given answer
set. In particular, each explanation graph will describe the derivation
of the truth value (i.e., true or false) of an atom using the
rules in the program. The explanation will also need to be flexible
enough to explain those contradictory situations, arising during the
construction of answer sets, where an atom is made true \emph{and}
false at the same time---for reference, these are the situations that
trigger a backtracking in systems like {\sc Smodels}~\cite{sim02}.

In the rest of this section, we will introduce this graph-based representation
of the support for the truth values of atoms in an interpretation.
In particular, we will incrementally develop this representation. We will start
with a generic graph structure (\emph{Explanation Graph}), which describes 
truth values without accounting for program rules. We will then identify specific
graph patterns that can be derived from program rules (\emph{Local Consistent
Explanations}), and impose them on the explanation graph, to obtain
the \emph{$(J,A)$-based Explanation Graphs}. These graphs are used to 
explain the truth values of an atom w.r.t. an interpretation $J$ 
and a set of assumptions $A$---where
an assumption is an atom for which we will not seek any explanations. The assumptions
derive from the inherent ``guessing'' process involved in the definition
of answer sets (and in their algorithmic construction), and they will be used
to justify atoms that have been ``guessed'' in the construction of the 
answer set and for which a meaningful explanation cannot be constructed. 

\smallskip

Before we proceed, let us introduce notation that will be used in 
the following discussion. 


For an atom $a$, we write $a^+$ to denote the fact that the atom $a$ is true, and 
$a^-$ to denote the fact that $a$ is false. We will call $a^+$ and $a^-$ the
\emph{annotated} versions of $a$. Furthermore, we will define
$atom(a^+) = a$ and $atom(a^-)=a$.
For a set of atoms $S$, we define the following sets of annotated atoms:
\begin{itemize}
\item $S^p = \{a^+ \mid a \in S\}$, 
\item $S^n = \{a^- \mid a \in S\}$.
\end{itemize}
Furthermore, we denote with $\naf S$ the set $\naf S  = \{\naf a \mid a \in S\}$. 

\subsection{Explanation Graphs}

In building the notion of justification, we 
will start from a very
general (labeled, directed) graph structure, called \emph{explanation graph}.
We will
incrementally construct the notion of justification, by progressively
adding the necessary restrictions to it. 

\begin{definition}
[Explanation Graph]
\label{egraph}
For a program $P$, an \emph{explanation
graph} (or \emph{e-graph}) is a
 labeled, directed graph $(N,E)$, where
$N \subseteq {\cal A}^p \cup {\cal A}^n \cup \{\textit{assume},\top,\bot\}$ and
$E \subseteq N \times N \times \{+,-\}$, which satisfies the following
properties:
\begin{enumerate}
\item\label{one1} the only sinks in the graph are: $assume$, $\top$, and $\bot$; 

\item\label{two1} for every $b \in N \cap {\cal A}^p$, we have that 
$(b,assume,-) \not\in E$ and $(b,\bot,-) \not\in E$;

\item\label{three1} for every $b \in N \cap {\cal A}^n$, we have that
$(b,assume,+) \not\in E$ and $(b,\top,+) \not\in E$;

\item\label{four1} for every $b \in N$, if $(b,l,s) \in E$ for some 
$l \in \{assume, \top, \bot\}$ and $s \in \{+,-\}$
then $(b,l,s)$ is the only outgoing  edge originating from $b$.
\end{enumerate}
\end{definition}
Property (\ref{one1}) indicates that each atom appearing in an e-graph
should have outgoing edges (which will explain the truth value of the atom). 
Properties (\ref{two1}) and (\ref{three1}) ensure that true (false) atoms are not
explained using explanations that are proper for false (true) atoms.
Finally, property (\ref{four1}) ensures that atoms explained using the special
explanations $assume$, $\top$, $\bot$ have only one explanation in the graph. Intuitively,
\begin{itemize}
\item $\top$ will be employed to explain program facts---i.e., their truth does not depend
	on other atoms;
\item $\bot$ will be used to explain atoms that do not have defining rules---i.e., the falsity
	is not dependent on other atoms; and
\item $assume$ is used for atoms we are not seeking any explanations for.
\end{itemize}
Each edge of the graph connects two annotated atoms 
or an annotated atom with one of the nodes 
in $\{\top, \:\bot, \: assume\}$, and it is marked by
a label from $\{+,-\}$.
Edges labeled $'+'$ are called 
\emph{positive} edges, while those labeled $'-'$ are called 
\emph{negative} edges. 
A path in an e-graph is \emph{positive} if it contains only 
positive edges, while a path is negative if it contains at least
 one negative edge. We will denote with $(n_1,n_2) \in E^{*,+}$ 
the fact that there is a positive path in the e-graph from $n_1$ to $n_2$.

\begin{example}\label{ex1}
Figure~\ref{img1} illustrates several simple e-graphs. 
\begin{figure}[htbp]
\centerline{\psfig{figure=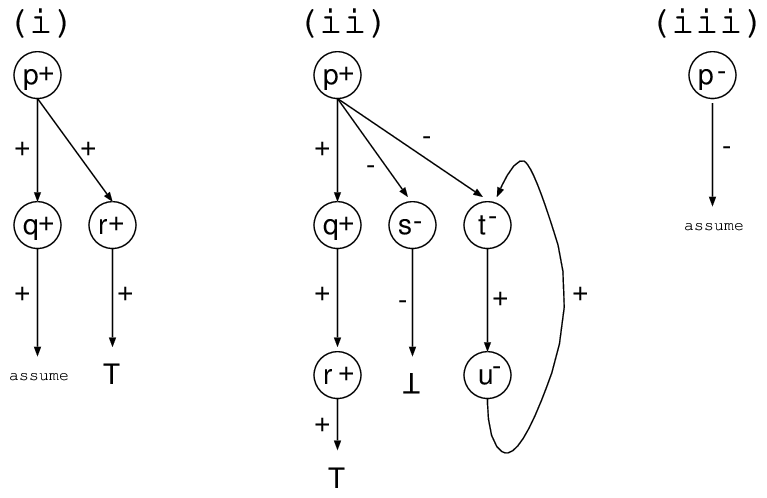,width=.65\textwidth}}
\caption{Simple e-graphs}\label{img1}
\end{figure}
Intuitively,
\begin{itemize}
\item The graph {\tt (i)} describes the true state of $\texttt{p}$
	by making it positively dependent on the true state of $\texttt{q}$ and
	$\texttt{r}$; in turn, $\texttt{q}$ is simply assumed to be true while
	$\texttt{r}$ is a fact in the program.
\item The graph {\tt (ii)} describes more complex dependencies; in particular,
	observe that $\texttt{t}$ and $\texttt{u}$ are both false and they are
	mutually dependent---as in the case of a program containing the rules
	\[\begin{array}{lclclcl}
	 \texttt{t} & \uffa & \texttt{u}. & \hspace{1cm} & 
         \texttt{u} &\uffa & \texttt{t}.
	  \end{array}
	\]
	Observe also that $\texttt{s}$ is explained being false because there
	are no rules defining it.
\item The graph {\tt (iii)} states that $\texttt{p}$ has been simply 
	assumed to be false.
\end{itemize}
\hfill $\Box$
\end{example}

Given an explanation graph and an atom, we can extract from the graph the 
elements that directly contribute to the truth value of the atom. We will 
call this set of elements the support of the atom. This is formally defined
as follows.

\begin{definition}
Let $G = (N,E)$ be  an e-graph and 
$b \in N \cap ({\cal A}^p\cup{\cal A}^n)$ a node in $G$. The direct support 
of $b$ in $G$, denoted by $support(b,G)$, is defined as follows. 
\begin{itemize}
\item $support(b,G) = \{atom(c) \mid (b,c,+) \in E\} \cup  
\{\naf atom(c) \mid (b,c,-) \in E\}$, if for every $\ell \in \{assume, \top, \bot\}$
and $s \in \{+,-\}$, $(b,\ell,s) \not\in E$;

\item $support(b,G) = \{\ell\}$ if $(b,\ell,s) \in E$,
$\ell \in \{assume, \top, \bot\}$ and $s \in \{+,-\}$.
\end{itemize}
\end{definition}

\begin{example}
If we consider the e-graph {\tt (ii)} in Figure~\ref{img1}, then we have
that $support(\texttt{p}^+, G_2) = \{\texttt{q}, \naf\ \texttt{s}, \naf\ \texttt{t}\}$ while
$support(\texttt{t}^-,G_2) = \{\texttt{u}\}$.

We also have $support(\texttt{p}^+, G_1) = \{\texttt{q, r}\}$.
\hfill$\Box$
\end{example}
It is worth mentioning that an explanation graph is a general concept 
aimed at providing arguments for answering the question `{\em why 
is an atom true or false?}' In this sense, it is similar 
to the concept of a support graph used in program analysis 
\cite{SahaR05}. The main difference between these two concepts lies in 
that support graphs are defined only for definite programs while 
explanation graphs are defined for general logic programs. 
Furthermore, a support graph contains information about the support 
for {\em all} answer while an explanation graph stores only 
the support for {\em one} atom. An explanation graph 
can be used to answer the question of why 
an atom is false which is not the case for support graphs. 

\subsection{Local Explanations and $(J,A)$-based Explanations}
The next step towards the definition of the concept of justification
requires enriching the general concept of e-graph with explanations
of truth values of atoms that are derived from the rules of the program.

A \emph{Local Consistent Explanation (LCE)}
 describes one step of justification for a literal.
Note that our notion of
local consistent explanation is similar in spirit, but different
in practice, from the analogous definition used in~\cite{PemmasaniGDRR04,RoychoudhuryRR00}. 
It describes the possible local reasons for the truth/falsity of a literal. 
If $a$ is true, the explanation contains those bodies of the rules for $a$ that
are satisfied by $I$. If $a$ is false, 
the explanation contains  sets of literals that are false in $I$ and they
 falsify all rules for $a$.

The construction of a LCE is performed w.r.t. a possible interpretation and
a set of atoms $U$---the latter contains atoms
that are automatically assumed to be false,
without the need of justifying them. 
The need for this last component (to be further
elaborated later in the paper) derives from the practice of 
computing answer sets, where the truth value of certain atoms 
is first guessed and then later verified. 

\begin{definition}[Local Consistent Explanation]\label{lcedef}
Let $P$ be a program, $b$ be an atom, $J$ a possible 
interpretation, $U$ a set of atoms (\emph{assumptions}),
and 
$S \subseteq {\cal A} \cup \naf{\cal A} \cup \{assume,\top,\bot\}$
a set of literals. 
We say that 
\begin{enumerate}
\item  $S$ is \emph{a} local consistent explanation of $b^+$ w.r.t. $(J,U)$, 
	if $b \in J^+$ and 
\begin{itemize}
\item[{$\circ$}] $S = \{assume\}$,  or
\item[{$\circ$}] 
	 $S \cap {\cal A} \subseteq J^+$,
	$\{c \mid \naf c \in S\} \subseteq J^- \cup U$, and  
	there is a rule $r$ in $P$ such that $head(r) = b$ and
	$S = body(r)$; for convenience, we write $S = \{\top\}$ to denote 
	the case where $body(r) = \emptyset$. 
\end{itemize}

\item $S$ is a local consistent explanation of $b^-$ w.r.t. $(J,U)$ 
	if  $b \in J^- \cup U$ and
\begin{itemize}
\item[{$\circ$}] $S = \{assume\}$; or 
\item[{$\circ$}] 
	 $S \cap {\cal A} \subseteq J^- \cup U$,
	$\{c \mid \naf c \in S\} \subseteq J^+ $, and 
	$S$ is a minimal set of literals such that 
	for every rule $r \in P$, if $head(r) = b$, then
	$pos(r) \cap S  \ne \emptyset$ or 
	$neg(r) \cap \{c \mid \naf c \in S\} \ne \emptyset$;
	for convenience, we write $S = \{\bot\}$ to denote 
the case $S = \emptyset$. 
\end{itemize}
\end{enumerate}
We will denote with $LCE^p_P(b,J,U)$ the set of all the LCEs of $b^+$ w.r.t.
$(J,U)$, and with $LCE^n_P(b,J,U)$ the set of all the LCEs of $b^-$ w.r.t.
$(J,U)$. 
\end{definition}
Observe that $U$ is the set of atoms that are assumed to be false. 
For this reason, negative LCEs are defined for 
elements $J^- \cup U$ but positive LCEs are defined only 
for elements in $J^+$. 
We illustrate this definition in a series of examples. 
\begin{example}\label{ex3}
Let $P$ be the program:
\[\begin{array}{lclclcl}
\texttt{p} & \uffa & \texttt{q},\:\: \texttt{r}. & \hspace{2cm} & \texttt{q}.  \\
\texttt{q} & \uffa & \texttt{r}. & & \texttt{r}.
  \end{array}
\]
The program admits only one answer set $M=\langle \{\texttt{p},\texttt{q},\texttt{r}\},\emptyset\rangle$. 
The LCEs for the atoms of this program w.r.t. $(M,\emptyset)$ are:
\[
\begin{array}{lcl}
LCE_P^p(\texttt{p},M,\emptyset) = \{ \{\texttt{q},\texttt{r}\},\{assume\} \} \\ 
	LCE_P^p(\texttt{q},M,\emptyset)= \{\{\top\}, \{\texttt{r}\},\{assume\}\}\\
LCE_P^p(\texttt{r},M,\emptyset) = \{\{\top\},\{assume\}\}
\end{array}
\]\hfill$\Box$
\end{example}
The above example shows a program with a unique answer set. 
The next example discusses the definition in a program with more than 
one answer set and an empty well-founded model. It also highlights the 
difference between the positive and negative LCEs for atoms 
given a partial interpretation and a set of assumptions. 
\begin{example}\label{ex6}
Let $P$ be the program:
\[\begin{array}{lclclcl}
\texttt{p} & \uffa & \naf \texttt{q}. & \hspace{2cm} & \texttt{q} & \uffa & \naf \texttt{p}.
  \end{array}
\]
Let us consider the partial interpretation $M = \langle\{\texttt{p}\},\emptyset\rangle$.
The following are LCEs w.r.t. $(M,\emptyset)$:
\[\begin{array}{l}
LCE_P^p(\texttt{p},M,\emptyset) = \{\{assume\}\} \\
	LCE_P^n(\texttt{q},M,\emptyset)= LCE_P^p(\texttt{q},M,\emptyset) =  \{\{\bot\}\}
  \end{array}
\]
The above LCEs are explanations for the truth value of $\texttt{p}$ 
and $\texttt{q}$ being true and false with respect to $M$ and the empty set of assumptions. 
Thus, the only explanation for $\texttt{p}$ being true is that it is assumed to 
be true, since the only way to derive $\texttt{p}$ to be true is to use the first rule
and nothing is assumed to be false, i.e., $not \:\texttt{q}$ is not true. 
On the other hand, $\texttt{q} \not\in M^- \cup \emptyset$ leads to the fact
that there is no explanation for $q$ being false. Likewise, 
because $\texttt{q} \not\in M^+$, there is no positive LCE for $\texttt{q}$ 
w.r.t. $(M,\emptyset)$.

The LCEs w.r.t. $(M,\{\texttt{q}\})$ are:
\[\begin{array}{l}
LCE_P^p(\texttt{p},M,\{\texttt{q}\}) = \{ \{assume\}, \{\naf \texttt{q}\}\}\\
	LCE_P^n(\texttt{q},M,\{\texttt{q}\}) = \{\{assume\}, \{\naf \texttt{p}\}\}
  \end{array}
\]
Assuming that $\texttt{q}$ is false leads to one additional explanation for 
$\texttt{p}$ being true.
Furthermore, there are now two explanations for $\texttt{q}$ being false. 
The first one is that it is assumed to be false 
and the second one satisfies the second condition in Definition~\ref{lcedef}. 

Consider the complete interpretation $M'=\langle \{\texttt{p}\},\{\texttt{q}\} \rangle$.
The LCEs w.r.t. $(M',\emptyset)$ are:
\[\begin{array}{l}
LCE_P^p(\texttt{p},M',\emptyset) = \{ \{assume\}, \{\naf \texttt{q}\}\}\\
	LCE_P^n(\texttt{q},M',\emptyset) = \{\{assume\}, \{\naf \texttt{p}\}\}
  \end{array}
\]
\hfill$\Box$
\end{example}
The next example uses a program with a non-empty well-founded model. 
\begin{example}\label{ex5}
Let $P$ be the program:
\[\begin{array}{lclclclclcl}
\texttt{a} & \uffa & \texttt{f},\: not\:\texttt{b}.& \hspace{1cm} &
\texttt{b} & \uffa & \texttt{e},\: not\:\texttt{a}. & \hspace{1cm} &
\texttt{e} . \\
\texttt{f} & \uffa & \texttt{e}. &&
\texttt{d} & \uffa & \texttt{c}, \:\texttt{e}. &&
\texttt{c} & \uffa & \texttt{d},\:\texttt{f}.
  \end{array}
\]
This program has the answer sets:
\[\begin{array}{lcr}
M_1 = \langle\{\texttt{f},\texttt{e},\texttt{b}\},\{\texttt{a},\texttt{c},\texttt{d}\}\rangle & \hspace{1cm}&
M_2 = \langle \{\texttt{f},\texttt{e},\texttt{a}\},\{\texttt{c},\texttt{b},\texttt{d}\}\rangle\end{array}\]
Observe that the well-founded model of this program is
$\langle W^+, W^- \rangle = \langle \{\texttt{e},\texttt{f}\}, \{\texttt{c},\texttt{d}\}\rangle$.
The following are LCEs w.r.t. the answer set $M_1$ and 
the empty set of assumptions (those for $(M_2,\emptyset)$ 
have a similar structure):
\[\begin{array}{l}
LCE_P^n(\texttt{a},M_1,\emptyset)=\{\{\naf \texttt{b}\},\{assume\}\} \\
LCE_P^p(\texttt{b},M_1,\emptyset)=\{\{\texttt{e},\naf \texttt{a}\},\{assume\}\}\\
LCE_P^p(\texttt{e},M_1,\emptyset) = \{\{\top\},\{assume\}\} \\
LCE_P^p(\texttt{f},M_1,\emptyset)=\{\{\texttt{e}\},\{assume\}\} \\
LCE_P^n(\texttt{d},M_1,\emptyset) = \{\{\texttt{c}\},\{assume\}\} \\
LCE_P^n(\texttt{c},M_1,\emptyset)=\{\{\texttt{d}\},\{assume\}\}
  \end{array}
\]
\hfill$\Box$
\end{example}

Let us open a brief parenthesis to discuss some complexity
issues related to the  existence 
of LCEs. 
First, checking whether or not there is 
a LCE of $b^+$ w.r.t. $(J,U)$ 
is equivalent to checking whether or not the program 
contains a rule $r$ whose head is $b$ and whose 
body is satisfied by the interpretation $\langle J^+,J^- \cup U\rangle$.
This leads to the following observation. 
\begin{observation}
Given a program $P$,  a possible interpretation $J$,
a set of assumptions $U$, and an atom $b$, determining 
whether or not there is a LCE $S$ of $b^+$ w.r.t. $(J,U)$ 
such that $S \ne \{assume\}$ can be done in time 
polynomial in the size of $P$. 
\end{observation}
In order to determine whether or not there exists a LCE of $b^-$ w.r.t. $(J,U)$,
we need to find a minimal set of literals $S$ that satisfies 
the second condition of Definition \ref{lcedef}. This can also be 
accomplished in  time polynomial in the size of $P$. In fact,
let $P_b$ be the set of rules in $P$ whose 
head is $b$. Furthermore, for a rule $r$, let 
$$S_r(J,U) = \{a \mid a \in pos(r) \cap (J^- \cup U)\} 
	\cup \{not \: a \mid a \in J^+ \cap neg(r)\}.$$
Intuitively, $S_r(J,U)$ is the maximal set of literals that 
falsifies the rule $r$ w.r.t. $(J,U)$. 
To find a LCE for $b^-$, it is necessary to have 
$S_r(J,U) \ne \emptyset$ for every $r \in P_b$. 
Clearly, computing $S_r(J,U)$ for $r \in P_b$
can be done in polynomial time in the size of $P$.
Finding a minimal set $S$ such that $S \cap S_r \ne \emptyset$
for every $r \in P_b$ can be done by scanning through the set
$P_b$ and adding to $S$ (initially set to $\emptyset$) 
an arbitrary element of $S_r(J,U)$ if $S \cap S_r(J,U) = \emptyset$. 
This leads to the following observation.
\begin{observation}
Given a program $P$,  a possible interpretation $J$,
a set of assumptions $U$, and an atom $b$, determining 
whether there exists a LCE $S$ of $b^-$ w.r.t. $(J,U)$ 
such that $S \ne \{assume\}$ can be done in time 
polynomial in the size of $P$. 
\end{observation}

\smallskip

We are now ready to instantiate the notion of e-graph by forcing the edges of
the e-graph to represent encodings of local consistent explanations of the corresponding
atoms. 
To select
an e-graph as an acceptable explanation, we need  two additional
components: the current interpretation ($J$) and the collection ($U$) of
elements that have been introduced in the interpretation without any
``supporting evidence''. An e-graph based on $(J,U)$ is defined next.

\begin{definition}[$(J,U)$-Based  Explanation Graph]
\label{ja-based}
Let $P$ be a program, $J$ a possible interpretation, 
$U$ a set of atoms, and 
$b$ an element in ${\cal A}^p \cup {\cal A}^n$. 
A $(J,U)$-\emph{based explanation graph} $G=(N,E)$ of $b$
is an e-graph such that 

\begin{itemize}

\item {\bf (Relevance)} every node $c \in N$ is reachable from $b$ 

\item  {\bf (Correctness)} for every $c \in N \setminus \{assume,\top,\bot\}$, $support(c,G)$ is 
an LCE of $c$ w.r.t. $(J,U)$
\end{itemize}

\end{definition}
The two additional conditions we impose on the e-graph force the graph to 
be connected w.r.t. the element $b$ we are justifying, and force the selected
nodes and edges to reflect local consistent explanations for the various elements.

The next condition we impose on the explanation graph is aimed at ensuring
that no positive cycles are present. The intuition is that atoms that are true
in an answer set should have a non-cyclic support for their truth values. Observe
that the same does not happen for elements that are false---as in the case of
elements belonging to unfounded sets~\cite{apt94a}.

\begin{definition}[Safety]
A $(J,U)$-based e-graph $(N,E)$ is {\em safe} if  $\forall b^+ \in N$, 
$(b^+,b^+)\not\in E^{*,+}$.
\end{definition}

\begin{example}
Consider the e-graphs in Figure \ref{fig1}, for the program of 
Example~\ref{ex5}. 

\smallskip

\noindent
Neither the e-graph of $\texttt{a}^+$ ({\tt (i)} nor the
e-graph {\tt (ii)}) is a 
$(M_1,\{\texttt{c},\texttt{d}\})$-based e-graph of 
$\texttt{a}^+$, since
$support(\texttt{b},G)=\{assume\}$ in both cases, and this does not
represent a valid LCE for $\texttt{b}^-$ 
(since $\texttt{b} \notin M_1^-\cup \{\texttt{c},\texttt{d}\}$). Observe,
on the other hand, that they are both acceptable
$(M_2,\{\texttt{b},\texttt{c},\texttt{d}\})$-based e-graphs of $\texttt{a}^+$.

\smallskip
\noindent
The e-graph of $\texttt{c}^+$ (the graph {\tt (iii)})
is neither a $(M_1,\{\texttt{c},\texttt{d}\})$-based nor 
a $(M_2,\{\texttt{b},\texttt{c},\texttt{d}\})$-based e-graph of $\texttt{c}^+$,
while the e-graph of $\texttt{c}^-$ (graph {\tt (iv)}) is  a 
$(M_1,\{\texttt{c},\texttt{d}\})$-based and
a $(M_2,\{\texttt{b},\texttt{c},\texttt{d}\})$-based e-graph of $\texttt{c}^-$.

\smallskip
\noindent
Observe also that all the graphs are safe.\hfill$\Box$
\end{example}

\begin{figure}[htbp]
\centerline{\psfig{figure=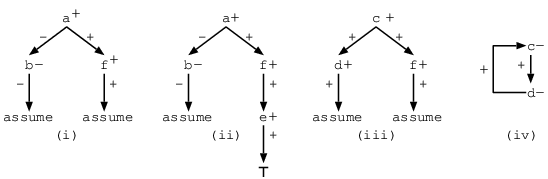,width=\textwidth}}
\caption{Sample $(J,U)$-based Explanation Graphs}
\label{fig1}
\end{figure}


\section{Off-Line Justifications}\label{off}
\emph{Off-line} justifications are employed to characterize the
``reason'' for the 
truth value of an atom w.r.t. a given answer set $M$. The definition will represent
a refinement of the $(M,A)$-based explanation graph, where $A$ will be selected
according to the properties of the answer set $M$. Off-line
justifications will rely on the assumption that $M$ is a \emph{complete} interpretation. 

Let us start with a simple observation.
If $M$ is an answer set of a program $P$, and $WF_P$ is the well-founded model of $P$, then
it is known that, $WF_P^+ \subseteq M^+$ and $WF_P^- \subseteq M^-$ \cite{apt94a}.
Furthermore, we observe that the content of $M$ 
is uniquely determined by the truth values assigned to 
the atoms in $V=NANT(P) \setminus (WF_P^+ \cup WF_P^-)$\/, i.e.,
the atoms that
\begin{itemize}
\item  appear in negative literals in the program, and 
\item their truth value is not determined by the well-founded model.
\end{itemize}
We are interested in the subsets of $V$
with the following property: if all the elements in 
the subset are assumed to be false, then the truth value
of all other atoms
in $\cal A$ is uniquely determined and leads to the desired
answer set.
We call these subsets the \emph{assumptions} of the answer set.
Let us characterize this concept more formally.

\begin{definition}[Tentative Assumptions]
Let $P$ be a program and $M$ be an answer set of $P$. 
The \emph{tentative assumptions} of $P$ w.r.t. $M$ (denoted by ${\cal TA}_P(M)$) 
are defined as: 
\[ {\cal TA}_P(M) = \{ a \:|\: a \in NANT(P)\:\wedge\: a \in M^- \:\wedge\: 
				a \not\in (WF_P^+\cup WF_P^-)\} \]
\end{definition}

The negative reduct of a program $P$ w.r.t. a set of atoms $U$ is a
program obtained from $P$ by forcing all the atoms in $U$ to be
false.

\begin{definition}[Negative Reduct]\label{nred}
Let $P$ be a program, $M$ an answer set of $P$, and $U \subseteq {\cal TA}_P(M)$ a set of 
tentative assumption atoms. 
The \emph{negative reduct} of $P$ w.r.t. $U$, denoted by $NR(P,U)$, 
is the set of rules: 
\[ NR(P,U) = P \setminus \{\: r \:|\: head(r) \in U\}. \] 
\end{definition}

\begin{example}
Let us consider the program
\begin{verbatim}
  p :- not q.                q :- not p.
  r :- p, s.                 t :- q, u.
  s.
\end{verbatim}
The well-founded model for this program is $\langle \{\texttt{s}\},\{\texttt{u}\}\rangle$. The program
has two answer sets, $M_1=\langle \{\texttt{p,s,r}\},\{\texttt{t,u,q}\}\rangle$ and
$M_2 = \langle \{\texttt{q,s}\},\{\texttt{p,r,t,u}\}\rangle$.
The tentative assumptions for this program w.r.t. $M_1$
is the set $\{\texttt{q}\}$. If we consider the set $\{\texttt{q}\}$, the negative reduct
of the program is the set of rules
\begin{verbatim}
  p :- not q.
  r :- p, s.        t :- q, u.
  s.
\end{verbatim}
\hfill$\Box$
\end{example}

We are now ready to introduce the proper concept of assumptions---these are 
those tentative assumptions that are sufficient to allow the reconstruction of
the answer set.

\begin{definition}[Assumptions]\label{assumption}
Let $P$ be a program and $M$ be an answer set of $P$. 
An \emph{assumption} w.r.t. $M$ is a set of atoms $U$
satisfying the following properties: 
\begin{itemize}
\item[(1)] $U \subseteq {\cal TA}_P(M)$, and 
\item[(2)] the well-founded model of $NR(P,U)$ is equal to $M$---i.e.,
\[ WF_{NR(P,U)} = M.\]
\end{itemize}
We will denote with $Assumptions(P,M)$ the 
set of all assumptions of $P$ w.r.t. $M$.
A \emph{minimal assumption} is an assumption that is minimal w.r.t. 
the set inclusion operator. 
We will denote with $\mu Assumptions(P,M)$ the 
set of all the minimal assumptions of $P$ w.r.t. $M$.
\end{definition}

An important observation we can make is that $Assumptions(P,M)$ is not an empty set, since
the complete set ${\cal TA}_P(M)$ is an assumption.

\begin{proposition}
\label{prop1} 
Given a program $P$ and an answer set $M$ of $P$, the well-founded model of
the program 
$NR(P,{\cal TA}_P(M))$ is equal to $M$. 
\end{proposition} 
\begin{proof}
Appendix A. 
\end{proof}

\begin{example}
Let us consider the program of Example~\ref{ex6}. The interpretation
$M= \langle \{\texttt{p}\}, \{\texttt{q}\}\rangle$ is an answer set. For this program we have:
\[\begin{array}{lcl}
	WF_P  & = & \langle \emptyset, \emptyset \rangle\\
	{\cal TA}_P(\langle \{\texttt{p}\}, \{\texttt{q}\}\rangle ) & = & \{\texttt{q}\}
  \end{array}
\]
Observe that $NR(P,\{\texttt{q}\}) = \{ \texttt{p} \uffa not\: \texttt{q}\}$. The well-founded model 
of this program is $\langle \{\texttt{p}\}, \{\texttt{q}\}\rangle$, which is equal to $M$. Thus,
$\{\texttt{q}\}$ is an assumption of $P$ w.r.t. $M$. In particular, one can see that this
is the only assumption we can have.\hfill$\Box$
\end{example}

\begin{example} \label{ex-wf}
Let us consider the following program $P$:
\[\begin{array}{lclclclclcl}
\texttt{a} & \uffa & \texttt{f}, \:\:not\:\texttt{b}. & \hspace{.5cm} & \texttt{b} & \uffa & \texttt{e}, \:\:not\:\texttt{a}. & \hspace{.5cm} & \texttt{e}  .\\
\texttt{f} & \uffa & \texttt{e}. & & \texttt{d} & \uffa & \texttt{c},\:\: \texttt{e}. & & \texttt{c} & \uffa & \texttt{d}, \:\texttt{f}, \:not\: \texttt{k}.\\
\texttt{k} & \uffa & \texttt{a}.
  \end{array}
\]
 The interpretation
$M_1=\langle \{\texttt{f},\texttt{e},\texttt{b}\}, \{\texttt{a},\texttt{c},\texttt{d},\texttt{k}\}\rangle$ is an answer set of the program. In particular:
\[\begin{array}{lcl}
	WF_P & = & \langle \{\texttt{e},\texttt{f}\}, \{\texttt{d},\texttt{c}\}\rangle\\
	{\cal TA}_P(\langle \{\texttt{f},\texttt{e},\texttt{b}\},\{\texttt{a},\texttt{c},\texttt{d}\}\rangle) & = & \{\texttt{a},\texttt{k}\}
  \end{array}
\]
The program $NR(P,\{\texttt{a}\})$ is:
\[\begin{array}{lclclcl}
	\texttt{b} & \uffa & \texttt{e},\: not\: \texttt{a}. & \hspace{1cm} & \texttt{e} .\\
	\texttt{f} & \uffa & \texttt{e}. & & \texttt{d} & \uffa & \texttt{c}, \:\texttt{e}.\\
	\texttt{c} & \uffa & \texttt{d}, \:\texttt{f},\: not\:\texttt{k}. & & \texttt{k} & \uffa & \texttt{a}.
  \end{array}
\]
The well-founded model of this program is $\langle \{\texttt{e},\texttt{f},\texttt{b}\},\{\texttt{a},\texttt{c},\texttt{d},\texttt{k}\}\rangle$.  Thus,
$\{\texttt{a}\}$ is an assumption w.r.t. $M_1$.

Observe also that if we consider $NR(P,\{\texttt{a},\texttt{k}\})$ 
\[\begin{array}{lclclcl}
	\texttt{b} & \uffa & \texttt{e},\: not\: \texttt{a}. & \hspace{1cm} & \texttt{e} .\\
	\texttt{f} & \uffa & \texttt{e}. & & \texttt{d} & \uffa & \texttt{c},\:\texttt{e}.\\
	\texttt{c} & \uffa & \texttt{d},\texttt{f},\: not\:\texttt{k}. 
  \end{array}
\]
The well-founded model of this program is also $\langle \{\texttt{e},\texttt{f},\texttt{b}\},\{\texttt{a},\texttt{c},\texttt{d},\texttt{k}\}\rangle$, thus
making $\{\texttt{a},\texttt{k}\}$ another assumption. 
Note that this second assumption is not minimal.\hfill$\Box$
\end{example}

We will now specialize e-graphs to the case of answer sets, where
only false elements can be used as assumptions.

\begin{definition}[Off-line Explanation Graph]
Let $P$ be a program, $J$ a partial interpretation,  $U$ a set of atoms, and 
$b$ an element in ${\cal A}^p \cup {\cal A}^n$. An \emph{off-line explanation graph} $G=(N,E)$ of $b$
w.r.t. $J$ and $U$ is a $(J,U)$-based e-graph of $b$ satisfying the following
additional conditions:
\begin{itemize}
\item[{$\circ$}] there exists no $p^+ \in N$ such that $(p^+,\textit{assume},+) \in E$; and
\item[{$\circ$}] $(p^-,\textit{assume},-) \in E$ iff $p \in U$. 
\end{itemize}
We will denote with ${\cal E}(b,J,U)$  the set of all off-line explanation graphs
of $b$ w.r.t. $J$ and $U$.
\end{definition}
The first condition ensures that true elements cannot be treated as assumptions, while
the second condition ensures that only assumptions are justified as such in the graph.

\begin{definition}[Off-line Justification]
Let $P$ be a program, $M$ an answer set,  $U \in Assumptions(P,M)$, and 
$a \in {\cal A}^p\cup {\cal A}^n$. An \emph{off-line justification} of $a$
w.r.t. $M$ and $U$  is an element $(N,E)$
of ${\cal E}(a,M,U)$ which is safe.

If $M$ is an answer set and $x\in M^+$ (resp. $x \in M^-$), then $G$ is an off-line
justification of $x$ w.r.t. $M$ and the assumption $U$ iff $G$ is an
off-line justification of $x^+$ (resp. $x^-$) w.r.t. $M$ and $U$.
\end{definition}

\begin{example}\label{ex2}
Let us consider the program in Example \ref{ex5}.
We have that $NANT(P) = \{b, a\}$. 
The assumptions for this program are:
$$Assumptions(P,M_1)   =   \{ \{a\} \}\:\:\:\textit{ and }\:\:\:
Assumptions(P,M_2)  =  \{ \{b\} \}.$$ 
The off-line justifications for atoms in $M_1$ 
w.r.t. $M_1$ and $\{a\}$ are shown in Figure~\ref{fig2}.

\begin{figure}[htb] 
\centerline{\psfig{figure=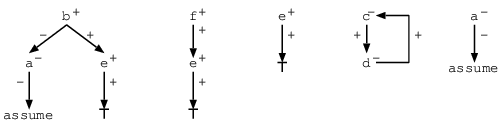,width=0.8\textwidth}}
\caption{Off-line justifications w.r.t. $M_1$ and $\{a\}$ for $b^+$, 
$f^+$, $e^+$, $c^-$ and $a^-$ (left to right)}
\label{fig2}
\end{figure}

\end{example}
Justifications are built by assembling items
from the LCEs of the various atoms and  avoiding the creation
of positive cycles in the justification of true atoms. Also,
the justification is built w.r.t. a chosen set of assumptions ($A$), 
whose elements are all
assumed false.

In general, an atom may admit multiple
justifications, even w.r.t. the same assumptions.
The following lemma shows that elements in $WF_P$ can be justified
without negative cycles and  assumptions. 

\begin{lemma}\label{good}
Let $P$ be a program, $M$ an answer set, and $WF_P$ the well-founded model
of $P$. 
Each atom $a\in WF_P$ has a justification 
w.r.t. $M$ and $\emptyset$ which does not contain any negative cycle.
\end{lemma}
From the definition of assumption and from the previous lemma we can infer
that a justification free of negative cycles can be built for every atom.

\begin{proposition}\label{propimp}
Let $P$ be a program and $M$ an answer set. For each atom $a$,
there is  an off-line 
justification w.r.t. $M$ and $M^-\setminus WF_P^-$
which does not contain negative cycles.
\end{proposition}
Proposition \ref{propimp} underlines an important property---the fact that all true
elements can be justified in a non-cyclic fashion.  This makes the justification more
natural, reflecting the non-cyclic process employed in constructing the minimal 
answer set (e.g., using the iterations of $T_P$) and the well-founded model
(e.g., using the characterization in~\cite{BrassDFZ01}). This also gracefully
extends a similar  property satisfied by the justifications under well-founded
semantics used in~\cite{RoychoudhuryRR00}. Note that the only
cycles possibly present in the justifications are positive cycles associated 
to (mutually dependent) false elements---this is an unavoidable situation due the 
semantic characterization in well-founded and answer set semantics (e.g., unfounded sets).
A similar design choice has been made in~\cite{PemmasaniGDRR04,RoychoudhuryRR00}.

\begin{example}
Let us reconsider the following program $P$ from Example \ref{ex-wf}:
\[\begin{array}{lclclclclcl}
\texttt{a} & \uffa & \texttt{f}, \:\:not\:\texttt{b}. & \hspace{.5cm} & \texttt{b} & \uffa & \texttt{e}, \:\:not\:\texttt{a}. & \hspace{.5cm} & \texttt{e}  .\\
\texttt{f} & \uffa & \texttt{e}. & & \texttt{d} & \uffa & \texttt{c},\:\: \texttt{e}. & & \texttt{c} & \uffa & \texttt{d}, \:\texttt{f}, \:not\: \texttt{k}.\\
\texttt{k} & \uffa & \texttt{a}.
\end{array}
\]
and the answer set 
$M=\langle \{\texttt{f},\texttt{e},\texttt{b}\}, \{\texttt{a},\texttt{c},\texttt{d},\texttt{k}\}\rangle$ is an answer set of the program. 
The well-founded model of this program is 
\[
	WF_P  =  \langle \{\texttt{e},\texttt{f}\}, \{\texttt{d},\texttt{c}\}\rangle
\]
$\texttt{a}$ and $\texttt{k}$ are assumed to be false. Off-line justifications 
for $\texttt{b}^+, \texttt{f}^+, \texttt{e}^+$ and for 
$\texttt{c}^-, \texttt{d}^-, \texttt{a}^-$ with respect to $M$ and 
$M^- \setminus WF_P^- = \{\texttt{a}, \texttt{k}\}$, which do not 
contain negative cycles, are the same as those depicted 
in Figure~\ref{fig2}. $\texttt{k}^-$ has an off-line justification in which 
it is connected to $assume$ by a negative edge, as it is assumed to be false. 
\hfill $\Box$
\end{example}

\section{On-Line Justifications for ASP}
Off-line justifications provide a ``declarative trace'' for the truth values
of the atoms present in an answer set. The majority of the inference engines
for ASP construct answer sets in an incremental fashion, making choices
(and possibly undoing them) and declaratively applying the rules in the
program. Unexpected results (e.g., failure to produce any answer sets) require
a more refined view of computation. One way to address this problem is to 
refine the notion of justification to make possible the 
``declarative tracing'' of atoms w.r.t. a partially constructed 
interpretation. This is similar to debugging of imperative languages,
where breakpoints can be set and the state of the execution explored
at any point during the computation.
In this section, we introduce the concept of \emph{on-line justification},
which is generated \emph{during} the computation of an answer set and 
allows us to justify atoms w.r.t. an incomplete interpretation---that 
represents an intermediate
step in the construction of the answer set.

\subsection{Computation} \label{subsec-comp}
The concept of on-line justification is applicable to computation
models that construct answer sets in an incremental fashion, e.g.,
{\sc Smodels} and {\tt DLV} 
\cite{sim02,eiter98a,GebserKNS07,AngerGLNS05}. We can  view the computation
as a sequence of steps, each associated to a partial interpretation.
We will focus, in particular, on computation models where the
progress towards the answer set is monotonic.

\begin{definition}[General Computation]\label{genc}
Let $P$ be a program. 
A \emph{general computation} is a sequence 
$M_0, M_1, \dots, M_k$,
such that
\begin{description}
\item[\emph{(i)}] $M_0 = \langle \emptyset, \emptyset\rangle$,
\item[\emph{(ii)}] $M_0, \dots, M_{k-1}$ are partial interpretations, and
\item[\emph{(iii)}] $M_{i} \sqsubseteq M_{i+1}$ for $i=0,\dots,k-1$.
\end{description}
A \emph{general complete computation} is a computation 
$M_0, \dots, M_k$ such that $M_k$ is an answer set of $P$.
\end{definition}
In general, we do not require $M_k$---the ending point of the
computation---to be a partial interpretation, since we wish to 
model computations that can also ``fail''---i.e., $M_k^+ \cap M_k^- \neq \emptyset$. 
This is, for example, what might happen during a {\sc Smodels} computation---whenever
the {\tt Conflict} function succeeds~\cite{sim02}.

We will refer to a pair of sets of atoms as a  
{\em possible interpretation} (or {\em p-interpretation} for short). 
Clearly, each partial interpretation is a p-interpretation, but not vice versa. 
Abusing the notation, we use $J^+$ and $J^-$ to indicate the first
and second component of a p-interpretation $J$; moreover, 
$I \sqsubseteq J$ denotes that $I^+ \subseteq J^+$ and $I^- \subseteq J^-$.

Our objective is to associate a form of justification to each intermediate step
$M_i$ of a general computation. Ideally, we would like the 
justifications associated to each $M_i$ to explain truth values in the 
``same way'' as in the final off-line justification.
Since the computation model might rely on 
guessing some truth values, $M_i$ might not contain sufficient
information to develop a valid justification for each element in $M_i$.
We will identify those atoms 
for which a justification can be constructed given $M_i$. 
These atoms describe a p-interpretation $D_i \sqsubseteq M_i$. 
The computation of $D_i$ is defined based on the two operators, $\Gamma$ 
and $\Delta$, which will respectively compute $D_i^+$ and 
$D_i^-$.

\medskip

Let us start with some preliminary definitions.
Let $P$ be a program and $I$ be a p-interpretation. A set of atoms $S$ 
is called a {\em cycle w.r.t. I} if, for every $a \in S$ and for
each $r \in P$ 
such that $head(r) = a$, we have that one of the following holds:
\begin{itemize}
\item $pos(r) \cap I^- \neq \emptyset$ (rule is  falsified by $I$), or
\item $neg(r) \cap I^+ \neq \emptyset$ (rule is falsified by $I$),  or 
\item $pos(r) \cap S \ne \emptyset$ (rule is in a cycle with elements of $S$).
\end{itemize}

We can observe that, if $I$ is an interpretation, 
$S$ is a cycle w.r.t. $I$, and $M$ is an answer set with $I \sqsubseteq M$, 
then $S \subseteq M^-$---since the elements of $S$ are either falsified by 
the interpretation (and, thus, by $M$) or they are part of an unfounded set.

The set of cycles w.r.t. $I$ is denoted by $cycles(I)$. 
Furthermore, for every element $a \in {\cal A}^p \cup {\cal A}^n$,
let $PE(a,I)$ be the set of local consistent explanations of $a$ 
w.r.t. $I$ and $\emptyset$---i.e.,
LCEs that do not require any assumptions and that build on the interpretation
$I$.

We are now ready to define the operators that will compute the $D_i$ subset of the
p-interpretation $M_i$. 

\begin{definition}
Let $P$ be a program and $I \sqsubseteq J$ be two p-interpretations. We define 
\[
\begin{array}{lll}
\Gamma_I(J) &  = & I^+ \cup \{head(r) \in J^+ \mid r\in P, I \models body(r)\}   \\
\Delta_I(J) &  = & I^- \:\cup\: \{a \in J^-\mid PE(a^-,I) \ne \emptyset\} 
		\:\cup\: \bigcup \{S \mid S \in cycles(I), S \subseteq J^- \}  \\
\end{array}
\]
\end{definition}
Intuitively, for $I \sqsubseteq J$, $\Gamma_I(J)$ is a set of 
atoms that are true in $J$ and they will remain true
in every answer set extending $J$, if $J$ is a partial
interpretation.
The set $\Delta_I(J)$ contains atoms that are  false in $J$ and
in each answer set that extends $J$.
In particular, if
$I$ is the set of \emph{``justifiable''} atoms---i.e., atoms for which
we can construct a justification---and $J$ is 
the result of the current computation step, then we have  that 
$\langle \Gamma_I(J), \Delta_I(J) \rangle$ is a new interpretation satisfying 
the following two properties:
\begin{itemize} 
\item $I \sqsubseteq \langle \Gamma_I(J), \Delta_I(J) \rangle \sqsubseteq J$, and
\item it is possible to create a justification for all elements in $\langle \Gamma_I(J), \Delta_I(J) \rangle$.
\end{itemize}
Observe that it is not necessarily true that $\Gamma_I(J)=J^+$ and $\Delta_I(J)=J^-$.
This  means
that there may be elements in the current step of computation for which it is not possible (yet)
to construct a justification.
This reflects the practice  of guessing literals and propagating these guesses 
in the computation of answer sets, implemented by several solvers (based on variations of the
Davis-Putnam-Logemann-Loveland procedure~\cite{DavisLL62}).

\smallskip
 
We are now ready to specify how the set $D_i$ is computed. 
Let $WF_P = \langle W^+,W^- \rangle$ be the well-founded model of $P$ 
and let $J$ be a p-interpretation.\footnote{Remember that
$ {\cal TA}_P(J) = \{ a \:|\: a \in NANT(P)\:\wedge\: a \in J^- \:\wedge\: 
				a \not\in (WF_P^+\cup WF_P^-)\}$.}

\[
\begin{array}{lllllllll}
\Gamma^0(J) &  = & \Gamma_{\langle \emptyset,\emptyset\rangle}(J) & \hspace{.65cm}& & \Delta^0(J) & =  & {\cal TA}_P(J)
						  \cup \Delta_{\langle \emptyset, \emptyset \rangle} (J) \\

\Gamma^{i+1}(J) & = &\Gamma_{I_i}(J) & & & \Delta^{i+1}(J) & = & \Delta_{I_i}(J)\\
\multicolumn{3}{l}{\small (\textnormal{where } 
		I_i = \langle \Gamma^{i}(J), \Delta^{i}(J) \rangle)}\\
\end{array}
\]
Intuitively, 
\begin{enumerate}
\item The iteration process starts by collecting the facts of $P$ ($\Gamma^0$) and
	all those elements that are false either because there are 
	no defining rules for them or 	
	because they have been chosen to be false in the construction of $J$. All these
	elements can be easily provided with justifications.
\item The successive iterations expand the set of known justifiable elements from
	$J$ using $\Gamma$ and $\Delta$.
\end{enumerate}
Finally, we repeat the iteration process until a fixpoint is reached:
\[
\Gamma(J) = \bigcup_{i=0}^\infty \Gamma^i(J) \:\:\:\:\textnormal{ and } \:\:\:\:
\Delta(J) = \bigcup_{i=0}^\infty \Delta^i(J)
\]
Because $\Gamma^i(J) \subseteq \Gamma^{i+1}(J) \subseteq J^+ $ and 
$\Delta^i(J)  \subseteq \Delta^{i+1}(J) \subseteq J^-$ (recall that $I \sqsubseteq J$), 
we know that both $\Gamma(J)$ and $\Delta(J)$ are 
well-defined. We can prove the following: 

\begin{proposition}\label{gammadelta}
For a program $P$, we have that: 
\begin{itemize}
\item $\Gamma$ and $\Delta$ maintains the consistency of $J$, i.e., 
if $J$ is an interpretation, then $\langle \Gamma(J), \Delta(J) \rangle$ 
is also an interpretation;
\item $\Gamma$ and $\Delta$ are monotone w.r.t the argument $J$, i.e.,
if $J \sqsubseteq J'$ then $\Gamma(J) \subseteq \Gamma(J')$ and $\Delta(J) \subseteq \Delta(J')$;
\item $\Gamma(WF_P) = WF_P^+ $ and $\Delta(WF_P) = WF_P^-$; and 
\item If $M$ is an answer set of $P$, then 
	$\Gamma(M) = M^+ $ and $\Delta(M) = M^-$.
\end{itemize}
\end{proposition}
We next introduce the notion of  on-line explanation graph. 
\begin{definition}[On-line Explanation Graph]
Let $P$ be a program,  $A$ a set of atoms, $J$ a p-interpretation, and 
$a \in {\cal A}^p \cup {\cal A}^n$. 
An \emph{on-line explanation graph} $G=(N,E)$ of $a$
w.r.t. $J$ and $A$ is a $(J,A)$-based e-graph of $a$.
\end{definition}
In particular,  if $J$ is an answer set of $P$, then
 any off-line e-graph of $a$ w.r.t. $J$ and $A$ is 
	also an on-line e-graph of $a$ w.r.t. $J$ and $A$.

Observe that $\Gamma^0(J)$ contains the set of facts of $P$ that
 belongs to $J^+$, while  $\Delta^0(J)$ contains the set of atoms 
without defining rules and atoms belonging 
to positive cycles of $P$. As such, it is easy to see that, for each atom $a$ 
in $\langle \Gamma^0(J),  \Delta^0(J)\rangle$, we can construct an
e-graph for $a^+$ or $a^-$ whose nodes belong to $\Gamma^0(J) \cup \Delta^0(J)$.
Moreover:
	\begin{itemize}
	\item  if $a \in \Gamma^{i+1}(J) \setminus \Gamma^{i}(J)$, 
	then an e-graph with nodes (except $a^+$) belonging to $\Gamma^i(J) \cup \Delta^i(J)$
	can be constructed; 
	\item if $a \in \Delta^{i+1}(J) \setminus \Delta^{i}(J)$, 
	an e-graph with nodes (except $a^-$) belonging to $\Gamma^{i+1}(J) \cup \Delta^{i+1}(J)$
	can be constructed. 
	\end{itemize}
This leads to the following lemma.  

\begin{lemma}
\label{just-free}
Let $P$ be a program, $J$ a p-interpretation, and 
$A = {\cal TA}_P(J)$. The following properties hold:
\begin{itemize}
\item For each atom $a \in \Gamma(J)$ (resp. $a \in \Delta(J)$), 
there exists a \emph{safe} off-line e-graph of $a^+$ (resp. $a^-$) 
w.r.t. $J$ and $A$; 
\item for each atom $a \in J^+ \setminus \Gamma(J)$
(resp. $a \in J^- \setminus \Delta(J)$) there exists an on-line 
e-graph of $a^+$ (resp. $a^-$) w.r.t. $J$ and $A$. 
\end{itemize}
\end{lemma}
We will now discuss how the above proposition can be utilized in defining a notion 
called {\em on-line justification}. To this end, we associate to each partial 
interpretation $J$ a snapshot $S(J)$:
\begin{definition}\label{snapdef}
A \emph{snapshot} of a p-interpretation $J$ 
is a tuple $S(J) = \langle \textnormal{\em Off}(J), On(J), D \rangle$, where 

\begin{list}{$\bullet$}{\topsep=3pt \itemsep=4pt \parsep=0pt \leftmargin=12pt}
\item $D = \langle \Gamma(J), \Delta(J) \rangle$,

\item For each $a$ in $\Gamma(J)$,\\
	\centerline{$\textnormal{\em Off}(J)$ contains exactly one safe off-line e-graph of $a^+$
	w.r.t. $J$
	and ${\cal TA}_P(J)$;}

\item For each $a$ in $\Delta(J)$,\\
	\centerline{$\textnormal{\em Off}(J)$ contains exactly one safe off-line e-graph of
	$a^-$ w.r.t. $J$
	and ${\cal TA}_P(J)$;}

\item For each $a \in J^+\setminus\Gamma(J)$,\\
	 \centerline{$On(J)$ contains exactly one on-line  
	e-graph of $a^+$ w.r.t. $J$ and ${\cal TA}_P(J)$;}

\item For each $a \in J^-\setminus \Delta(J)$,\\
	 \centerline{$On(J)$ contains exactly one on-line  
	e-graph of $a^-$ w.r.t. $J$ and ${\cal TA}_P(J)$.}
\end{list}

\end{definition}

\begin{definition}[On-line Justification]
Given a computation $M_0, M_1, \dots, M_k$, an \emph{on-line justification} of the computation is
a sequence of snapshots $S(M_0), S(M_1), \dots, S(M_k)$.
\end{definition}

It is worth to point out that an on-line justification can 
be obtained in answer set solvers employing the computation model 
described in Definition \ref{genc}. This will be demonstrated in the next section 
where we discuss the computation of on-line justifications in the {\sc Smodels} 
system. We next illustrate the concept of an on-line justification. 

\begin{example}
Let us consider the program $P$ containing
\[\begin{array}{lclclclclcl}
\texttt{s} &\uffa& \texttt{a},\: not\: \texttt{t}. & \hspace{0.5cm} &       \texttt{a} &\uffa & \texttt{f},\: not\: \texttt{b}. & \hspace{0.5cm} &
\texttt{b} & \uffa & \texttt{e},\: not\: \texttt{a}. \\
 \texttt{e}. &  & & & \texttt{f} & \uffa & \texttt{e}.
  \end{array}
\]
Two possible general computations of $P$ are
\[
\begin{array}{llllllll} 
M^1_0 = \langle \{\texttt{e},\texttt{s}\}, \emptyset \rangle & \hspace{0.2cm}\mapsto \hspace{0.2cm}&
M^1_1 = \langle \{\texttt{e},\texttt{s},\texttt{a}\}, \{\texttt{t}\} \rangle &  \hspace{0.2cm}\mapsto \hspace{0.2cm} &
M^1_2 = \langle \{\texttt{e},\texttt{s},\texttt{a},\texttt{f}\}, \{\texttt{t},\texttt{b}\} \rangle   \\
M^2_0 = \langle \{\texttt{e},\texttt{f}\}, \emptyset \rangle & \multicolumn{1}{c}{\mapsto} & 
M^2_1 = \langle \{\texttt{e},\texttt{f}\}, \{\texttt{t}\} \rangle & \multicolumn{1}{c}{\mapsto} & 
M^2_2 = \langle \{\texttt{e},\texttt{f},\texttt{b},\texttt{a}\}, \{\texttt{t},\texttt{a},\texttt{b},\texttt{s}\} \rangle 
\end{array}
\]
The first computation is a complete computation leading to an answer set of $P$ while 
the second one is not. 

An on-line justification for the first computation is given next:
\[
\begin{array}{llllllll} 
S(M^1_0) & = & \langle X_0, Y_0, \langle \{\texttt{e}\},\emptyset \rangle \rangle \\
S(M^1_1) & = & \langle X_0 \cup X_1, Y_0 \cup Y_1, \langle \{\texttt{e}\},\{\texttt{t}\} \rangle \rangle \\
S(M^2_1) & = & \langle X_0 \cup X_1 \cup X_2, \emptyset, M_2^1 \rangle \\
\end{array}
\]
where (for the sake of simplicity we report only the edges of the graphs):
\[\begin{array}{lcl}
X_0 &= & \{ (\texttt{e}^+,\top,+)\} \\
Y_0 &=&  \{ (\texttt{s}^+,assume,+)\}\\
X_1 & = &  \{(\texttt{t}^-,\bot,-)\} \\
Y_1 &=& \{(\texttt{a}^+,assume,+)\}\\
X_2 &=& \{ (\texttt{f}^+,\texttt{e}^+,+), (\texttt{s}^+,\texttt{a}^+,+), (\texttt{s}^+,\texttt{t}^-,-), (\texttt{a}^+,\texttt{f}^+,+), (\texttt{a}^+,\texttt{b}^-,-), (\texttt{b}^-,assume,-)\}\\
  \end{array}
\]
An on-line justification for the second computation is:
\[\begin{array}{lcl}
S(M^2_0) & = & \langle X_0, Y_0, \langle \{\texttt{e},\texttt{f}\},\emptyset\rangle \rangle \\
S(M^2_1) & = & \langle X_0 \cup X_1, Y_0, \langle \{\texttt{e},\texttt{f}\},\{\texttt{t}\}\rangle \rangle \\
S(M^2_2) & = & \langle X_0 \cup X_1 \cup X_2, Y_0\cup Y_2, M_2^2 \rangle
  \end{array}
\]
where:
\[\begin{array}{lcl}
X_0 & = & \{ (\texttt{e}^+, \top,+), (\texttt{f}^+,\texttt{e}^+,+) \}\\
Y_0 & = & \emptyset\\
X_1 & = & \{ (\texttt{t}^-,\bot,-)\}\\
Y_1 & = & \emptyset\\
X_2 & = & \{ (\texttt{a}^+,\texttt{f}^+,+), (\texttt{a}^+,\texttt{b}^-,-), (\texttt{b}^+,\texttt{e}^+,+), (\texttt{b}^+,\texttt{a}^-,-)\}\\
Y_2 & = & \{ (\texttt{a}^-, assume,-), (\texttt{b}^-, assume,-)\}\\
  \end{array}
\]
\hfill$\Box$
\end{example}
We can relate the on-line justifications and off-line justifications as follows.

\begin{lemma}\label{conserve} 
Let $P$ be a program, $J$ an interpretation, and $M$ an answer set such 
that $J \sqsubseteq M$. For every atom $a$, if $(N,E)$ is a safe off-line e-graph 
of $a$ w.r.t. $J$ and $A$ where $A = J^- \cap {\cal TA}_P(M)$ then 
it is an off-line justification of $a$ w.r.t. $M$ and ${\cal TA}_P(M)$.
\end{lemma}
This leads to the following proposition.

\begin{proposition} \label{on-off}
Let $M_0, \ldots, M_k$ be a general complete computation and 
$S(M_0), \ldots, S(M_k)$ be an on-line justification of the computation.
Then, for each atom $a$ in $M_k$, the  e-graph of $a$ in $S(M_k)$ is 
an off-line justification of $a$ w.r.t. $M_k$ and ${\cal TA}_P(M_k)$.
\end{proposition}


\section{{\sc Smodels} On-line Justifications}\label{smo}
The notion of on-line justification presented in the previous section
is very general, to fit the needs of different answer set solver
implementations that follow the computation model presented in 
Subsection \ref{subsec-comp}. 
In this section, we illustrate how the  notion of on-line justification
has been specialized to (and implemented in)  a specific computation 
model---the one used in {\sc Smodels}
\cite{sim02}. This allows us to define an incremental version 
of on-line justification---where the specific steps 
performed by {\sc Smodels} are used to guide the incremental construction of 
the
justification. The
choice of {\sc Smodels} was dictated by availability
of its source code and its elegant design.

We begin with an overview of the algorithms employed by {\sc Smodels}. 
The
following description has been adapted from \cite{GiunchigliaM05,sim02}.
Although more abstract than the concrete implementation, and without various 
implemented features (e.g., heuristics, lookahead), it is sufficiently faithful to 
capture the spirit of our approach, and to guide the implementation
(see Section~\ref{imple}).

\subsection{An Overview of {\sc Smodels}' Computation}
We propose a description of the {\sc Smodels} algorithms based on
a composition of state-transformation operators.
In the following, we say that an interpretation $I$ does
not satisfy the body of a rule $r$ (i.e., $body(r)$ is false in $I$) if
$(pos(r) \cap I^- ) \cup (neg(r) \cap I^+) \ne \emptyset$.

\subsubsection*{\underline{\sc AtLeast} Operator:}
The $AtLeast$ operator
is used to expand a partial interpretation $I$ in such a way 
that each answer
set $M$ of $P$ that ``agrees'' with $I$---i.e., the elements in
$I$ have the same truth value in $M$ (or $I \sqsubseteq M$)---also 
agrees with the expanded interpretation. 

\noindent
Given a program $P$ and a partial interpretation $I$, we 
define the intermediate operators $AL_P^1, \dots, AL_P^4$ as follows:
\begin{itemize}
	\item {\bf Case 1.} if  $r\in P$, $head(r)\notin I^+$, $pos(P) \subseteq I^+$ and
			$neg(P) \subseteq I^-$ then
			\[\begin{array}{lcr} AL_P^1(I)^+=I^+\cup\{head(r)\} & \:\:\:\textit{and}\:\:\:&
			AL_P^1(I)^-= I^-\end{array}\]
	\item {\bf Case 2.}  if $a \notin I^+\cup I^-$ and 
			$\forall r\in P. (head(r)= a \Rightarrow \textit{$body(r)$ is false in $I$})$, 
			then
			$$AL_P^2(I)^+ = I^+\:\:\:\textit{ and }\:\:\: 
			AL_P^2(I)^- = I^- \cup \{a\}$$
	\item {\bf Case 3.}  if  $a\in I^+$ and $r$ is the only rule in $P$ 
			with $head(r)=a$ and whose body is not false in $I$ then
			\[AL_P^3(I)^+=I^+\cup pos(r) \:\:\:\textit{and}\:\:\: 
			AL_P^3(I)^- = I^- \cup neg(r)\]
	\item{\bf Case 4.}  if $a\in I^-$, $head(r) = a$, and 
\begin{itemize}
	\item if $pos(r) \setminus I^+ = \{b\}$ 
		then
			\[
			AL_P^4(I)^+ = I^+  
			\:\:\:\textit{and}\:\:\: 
			AL_P^4(I)^- = I^- \cup \{b\} 
			\]

	\item if $neg(r) \setminus I^+ = \{b\}$ 
		then
			\[
			AL_P^4(I)^+ = I^+ \cup \{b\}  
			\:\:\:\textit{and}\:\:\: 
			AL_P^4(I)^- = I^- 
			\]
\end{itemize}

\end{itemize}
Given a program $P$ and an interpretation $I$, 
$AL_P(I) = AL_P^i(I)$ if $AL_P^i(I) \neq I$ and $\forall j < i. \: AL_P^j(I)=I$ ($1 \leq i \leq 4$); otherwise,
$AL_P(I) = I$.

\subsubsection*{\underline{\sc AtMost} Operator:}
The $AtMost_P$ operator recognizes atoms that are defined exclusively
as mutual positive dependences (i.e., ``positive loops'')---and falsifies them.
Given a set of atoms  $S$, the operator $AM_P$ is defined as 
$AM_P(S) = S \cup \{head(r) \:|\: r\in P\wedge pos(r)\subseteq S\}$.

Given an interpretation $I$, the $AtMost_P(I)$ operator is defined as
$$
AtMost_P(I)^+  =   I^+ 
\:\:\:\textit{and}\:\:\: 
AtMost_P(I)^- =  I^- \cup \{p \in {\cal A} \:|\: 
	p \not\in \bigcup_{i\geq 0}S_i\}
$$ 
where $S_0   = I^+$ and $S_{i+1}  =  AM_P(S_i)$.

\subsubsection*{\underline{\sc Choose} Operator:}
This operator is used to randomly select an atom that is
unknown in a given  interpretation.
Given a partial interpretation $I$, $choose_P$
returns an atom of $\cal A$ such that 
$$ choose_P(I) \not\in I^+ \cup I^- \:\:\:\textnormal{ and }\:\:\: 
choose_P(I) \in NANT(P) \setminus (WF_P^+\cup WF_P^-).$$

\subsubsection*{\underline{\sc Smodels} Computation:}
Given  an interpretation $I$, we define the
transitions: 
\[\begin{array}{lclcl}
I & \mapsto_{AL^c}& I' & \hspace{.5cm} &\left[\begin{array}{l}
				\textit{If $I' = AL^c_P(I)$, $c \in \{1, 2, 3, 4\}$}
								      \end{array}\right.\\
&&\\
I & \mapsto_{atmost}& I' & \hspace{.5cm} &\left[\begin{array}{l}
				\textit{If $I'=AtMost_P(I)$}
					\end{array}\right.\\
&&\\
I & \mapsto_{choice}& I' & \hspace{.5cm} &\left[\begin{array}{l}
				\textit{\small If $I'=\langle I^+\cup\{choose_P(I)\}, I^-\rangle$ or }\\
				\textit{\small $I'=\langle I^+, I^-\cup\{choose_P(I)\}\rangle$} 
						\end{array}\right.
  \end{array}
\]
If there is an $\alpha$ in  $\{AL^1, AL^2, AL^3, AL^4, atmost, choice\}$
such that $I \mapsto_{\alpha} I'$, then we will simply denote this fact
with   $I \mapsto I'$. 


\begin{figure}[htb]
\begin{center}
\begin{minipage}[c]{.48\textwidth}
\begin{center}
\fbox{\begin{minipage}[t]{\textwidth}
{\footnotesize \begin{tabbing} 
ii\=iiii\=iiiii\=iiiii\=iiiii\=iiii\=iiiii\=iiiii\=iiiii\kill
\>{\bf function} \emph{smodels}($P$):\\
\> $S$ = $\langle \emptyset,\emptyset\rangle$;\\
\>{\bf loop}\\
\>\> $S$ = \emph{expand}($P$, $S$);\\
\>\> {\bf if} ($S^+ \cap S^- \neq \emptyset$) {\bf then} \\
\>\>\> {\bf fail};\\
\>\> {\bf if} ($S^+\cup S^- = {\cal A}$) {\bf then}\\
\>\>\> {\bf success}($S$);\\
\>\> {\bf pick either }\>\>\>\> \textit{\% non-deterministic choice}\\
\>\>\> $S^+$ = $S^+ \cup \{ choose(S) \}$ \textbf{  or} \\
\>\>\> $S^-$ = $S^- \cup \{ choose(S) \}$\\
\>{\bf endloop};
\end{tabbing}}
\end{minipage}}
\end{center}
\caption{Sketch of \emph{smodels}}
\label{main}
\end{minipage}
\hspace{.05\textwidth}
\begin{minipage}[c]{.4\textwidth}
\begin{center}
\fbox{\begin{minipage}[t]{\textwidth}
{\footnotesize
\begin{tabbing} 
ii\=iiii\=iiiii\=iiiii\=iiiii\=iiii\=iiiii\=iiiii\=iiiii\kill
\\
\>{\bf function} \emph{expand}($P$, $S$):\\
\>{\bf loop}\\
\>\> $S'$ = $S$;\\
\>\> {\bf repeat}\\
\>\>\> $S$ = $AL_P(S)$;\\
\>\> {\bf until} ($S$ = $AL_P(S)$);\\
\>\> $S$ = $AtMost$($P$, $S$);\\
\>\> {\bf if} ($S'$ = $S$) {\bf then return} ($S$);\\
\>{\bf endloop};\\
\end{tabbing}}
\end{minipage}}
\end{center}
\caption{Sketch of \emph{expand}}
\label{exp}
\end{minipage}
\end{center}
\end{figure}

The {\sc Smodels} system imposes constraints on the order of application 
of the transitions. Intuitively, 
the {\sc Smodels} computation is 
depicted in the algorithms of Figs. \ref{main} and \ref{exp}. 

We will need the following notations. 
A computation 
$I_0 \mapsto I_1\mapsto I_2 \mapsto \dots \mapsto I_n$
is said to be {\em $AL$-pure} if every transition 
in the computation is an $AL^c$ transitions and for 
every $c \in \{1,2,3,4\}$, $AL^c_P(I_n) = I_n$. 
A choice point of a computation 
$I_0 \mapsto I_1\mapsto I_2 \mapsto \dots \mapsto I_n$
is an index $1 \le j < n$ such that $I_j \mapsto_{choice} I_{j+1}$. 

\begin{definition}[{\sc Smodels} Computation]
Let $P$ be a program. Let 
$$C = I_0 \mapsto I_1\mapsto I_2 \mapsto \dots \mapsto I_n$$ 
be a computation and 
$$0 \le \nu_1 < \nu_2 < \dots < \nu_r< n$$ 
($r \geq 0$) be the sequence of all choice points 
in $C$.  We say that $C$ is a \emph{{\sc Smodels} computation}
if for every $0 \le j \le r$,
        there exists a sequence of indices  
		$\nu_j+1 = a_1 < a_2 < \ldots < a_t \le \nu_{j+1}- 1$ 
         ($\nu_{r+1}=n$ and $\nu_0=-1$) 
	such that 

		\begin{itemize}
		\item the transition $I_{a_{i+1}-1} \mapsto 
					I_{a_{i+1}}$  is 	
			an $\mapsto_{atmost}$ transition ($1\leq i \leq t-1$)
		\item the computation 
$I_{a_i} \mapsto \ldots \mapsto I_{a_{i+1}-1}$ 
is a $AL$-pure computation. 
		\end{itemize}

\end{definition}
We illustrate this definition in the next example. 

\begin{example}\label{exnew}
Consider the program of  Example~\ref{ex5}. A possible
computation of $M_1$ is:\footnote{We omit the steps
that do not change the interpretation.}
\[\begin{array}{lclclc}
\langle \emptyset,\emptyset \rangle  &\mapsto_{AL^1}& \langle \{ \texttt{e} \}, \emptyset \rangle & \mapsto_{AL^1} & \langle \{\texttt{e}, \texttt{f}\},\emptyset \rangle
&\mapsto_{atmost}\\
  \langle \{\texttt{e},\texttt{f}\}, \{\texttt{c},\texttt{d}\}\rangle& \mapsto_{choice} &\langle \{\texttt{e},\texttt{f},\texttt{b}\}, \{\texttt{c},\texttt{}d\}\rangle& 
			\mapsto_{AL^2}& \langle \{\texttt{e},\texttt{f},\texttt{b}\}, \{\texttt{c},\texttt{d},\texttt{a}\}\rangle 
  \end{array}
 \]
\hfill$\Box$
\end{example}

\subsection{{\sc Smodels} On-line Justifications}
We can use knowledge of the specific steps performed by {\sc Smodels} 
to guide the construction of an on-line justification. 

Assuming that 
$$C = M_0 \mapsto M_1\mapsto M_2 \mapsto \dots \mapsto M_n$$ 
is a computation of  {\sc Smodels}. 
Let $S(M_i) = \langle E_1, E_2, D\rangle$ 
and $S(M_{i+1}) = \langle E'_1, E'_2, D'\rangle$ be the snapshots correspond
to $M_i$ and $M_{i+1}$ respectively. 
Obviously, $S(M_{i+1})$ can be computed by the following steps:  

\begin{itemize} 
\item computing $D'= \langle \Gamma(M_{i+1}), \Delta(M_{i+1})\rangle$; 

\item updating $E_1$ and $E_2$ to obtain $E_1'$ and $E_2'$. 

\end{itemize}
We observe that $\langle \Gamma(M_{i+1}), \Delta(M_{i+1} \rangle$ 
can be obtained by computing the fixpoint of the $\Gamma$- and $\Delta$-function
with the starting value  $\Gamma_{\langle \Gamma(M_i),\Delta(M_i)\rangle}$
and $\Delta_{\langle \Gamma(M_i),\Delta(M_i)\rangle}$. This is possible 
due to the monotonicity of the computation. Regarding $E_1'$ and $E_2'$,
observe that the e-graphs for 
elements in $\langle \Gamma^k(M_{i+1}), \Delta^k(M_{i+1}) \rangle$ 
can be constructed using the e-graphs constructed for elements in 
 $\langle \Gamma^{k-1}(M_{i+1}), \Delta^{k-1}(M_{i+1}) \rangle$ 
and the rules involved in the computation of 
$\langle \Gamma^k(M_{i+1}), \Delta^k(M_{i+1}) \rangle$. Thus, we only need to 
update $E_1'$ with e-graphs of elements of 
$\langle \Gamma^k(M_{i+1}), \Delta^k(M_{i+1}) \rangle$
which do not belong to  $\langle \Gamma^{k-1}(M_{i+1}), \Delta^{k-1}(M_{i+1}) \rangle$.
Also, $E_2'$ is obtained from $E_2$ by removing the e-graphs of 
atoms that ``move'' into $D'$ and adding the e-graph 
$(a,assume,+)$ (resp. $(a,assume,-)$) for $a \in M_{i+1}^+$ 
(resp. $a \in M_{i+1}^-$) not belonging to $D'$. 
Clearly, this computation depends on the transition from $M_i$ to $M_{i+1}$. 
Assume that $M_i \mapsto_\alpha M_{i+1}$, the update of $S(M_i)$ to 
create $S(M_{i+1})$ is done as follows. 

\begin{itemize}
\item \fbox{\bf $\alpha \equiv choice$:} let $p$ be the atom chosen in this 
	step. 

	If $p$ is chosen to be true, then we can use the graph
		$$G_p = (\{a,\textit{assume}\}, \{(a,\textit{assume},+)\})$$
	and the resulting snapshot is
	{$ S(M_{i+1}) = \langle E_1, E_2 \cup \{G_p\}, D   \rangle$}. Observe
	that $D$ is
	unchanged, since the structure of the computation (in particular the
	fact that an \emph{expand} has been done before the choice) ensures
	that $p$ will not appear in the computation of $D$.

	If $p$ is chosen to be false, then we will need to add $p$ to $D^-$,
	compute $\Gamma(M_{i+1})$ and $\Delta(M_{i+1})$, and
	update $E_1$ and $E_2$ correspondingly; in particular, $p$ belongs 
	to $\Delta(M_{i+1})$ and 
	$G_p = (\{a,\textit{assume}\}, \{(a,\textit{assume},-)\})$ is
	added to $E_1$. 

\item \fbox{\bf $\alpha \equiv atmost$:} in this case, 
	$M_{i+1} = \langle M_i^+, M_i^-\cup AtMost(P,M_i)\rangle$.
	The computation of $S(M_{i+1})$ is performed as from 
	definition of on-line justification.
	In particular, observe that if
	$\forall c\in AtMost(P,M_i)$ we have that  $LCE_P^n(c,D)\neq \emptyset$
	then the computation can be started from 
	$\Gamma(M_i)$ and $\Delta(M_i)\cup AtMost(P,M_i)$.


\item \fbox{\bf $\alpha \equiv AL^1$:} let $p$ be the atom dealt with in this
	step and let $r$ be the rule employed. We have that 
	$M_{i+1} = \langle M_i^+\cup \{p\}, M_i^-\rangle$. If
	$D\models body(r)$ then $S(M_{i+1})$ will be computed
	starting from $\Gamma(M_i)\cup\{p\}$ and
	$\Delta(M_i)$. In particular, an off-line graph for $p$, let's say $G_p$,
 	 will be added to $E_1$, and such graph will be constructed using the
	LCE based on the rule $r$ and the e-graphs in $E_1$.

	 Otherwise, $S(M_{i+1}) = \langle E_1, E_2 \cup \{G^+(p,r,\Sigma)\}, D\rangle$,
	where $G^+(p,r,\Sigma)$ is an e-graph of $p^+$ constructed using the LCE
	of rule $r$ and the e-graphs in $\Sigma = E_1 \cup E_2$ 
	(note that all elements in $body(r)$ have an e-graph in $E_1\cup E_2$).

\item \fbox{\bf $\alpha \equiv AL^2$:} let $p$ be the atom dealt with 
	in this step. In this case $M_{i+1} = \langle M_i^+, M_i^-\cup \{p\}\rangle$.
	If there exists $\gamma \in LCE_P^n(p,D,\emptyset)$, then
	$S(M_{i+1})$ can be computed according to the definition of on-line
	justification, starting
	from $\Gamma(M_i)$ and $\Delta(M_i)\cup \{p\}$. Observe that
	the graph of $p$ can be constructed starting with 
	$\{(p,a,+)\mid a\in \gamma\}\cup \{(p,b,-)\mid not\:b\in \gamma\}$). 

	Otherwise, given
	an arbitrary $\psi \in LCE_P^n(p,M_i, \emptyset)$, we can construct
	an e-graph $G_p$ for $p^-$, such that $\psi = support(p^-,G_p)$,
	the graphs $E_1\cup E_2$ are used to describe the elements of $\psi$,
	and
	$S(M_{i+1}) = \langle E_1, E_2 \cup \{G_p\}, D\rangle$.


\item \fbox{\bf $\alpha \equiv AL^3$:} let $r$ be the rule used in this step and let
	$p = head(r)$. Then $M_{i+1} = \langle M_i^+ \cup pos(r), M_i^-\cup neg(r)\rangle$
	and $S(M_{i+1})$ is computed according to the definition of on-line justification.
	Observe that the e-graph $G_p$ for $p^+$ (added to $E_1$ or $E_2$) for
	$S(M_{i+1})$ will be constructed using $body(r)$ as $support(p^+,G_p)$, and 
	using the e-graphs in $E_1\cup E_2 \cup \Sigma$ for some 
	$$\Sigma \subseteq \{ (a^+,assume,+)\:|\:a\in pos(r)\} \cup \{(a,\textit{assume},-)\mid a\in neg(r)\}.$$


\item \fbox{\bf $\alpha \equiv AL^4$:} let $r$ be the rule processed and let 
	$b$ the atom detected in the body. If $b \in pos(r)$, then
	$M_{i+1} = \langle M_i^+, M_i^- \cup \{b\}\rangle$, while if
	$b \in neg(r)$
	then $M_{i+1} = \langle M_i^+\cup\{b\}, M_i^-\rangle$. In either cases,
	the snapshot 
	$S(M_{i+1})$ will be computed using the definition of on-line justification.
\end{itemize}

\begin{example}\label{ex44}
Let us consider the 
computation of Example~\ref{exnew}. A 
sequence of snapshots is (we provide only the edges of the graphs and we
combine together e-graphs of different atoms):

\[
\begin{array}{|r||c|c|c|}
\hline
    &        E_1      &      E_2      &      D\\
\hline
S(M_0) & \emptyset       &    \emptyset & \emptyset\\
S(M_1) &  \{(e^+,\top,+)\} &    \emptyset & \langle \{e\},\emptyset\rangle\\
S(M_2) & \{(e^+,\top,+), (f^+,e^+,+)\} & \emptyset  & \langle \{e,f\},\emptyset \rangle\\
S(M_3) & \left\{\begin{array}{c}
	(e^+,\top,+),\{f^+,e^+,+)\\
	(d^-,c^-,+), (c^-,d^-,+)\end{array}\right\} & \emptyset & \langle \{e,f\}, \{c,d\}\rangle\\
S(M_4) & \left\{\begin{array}{c}
	(e^+,\top,+),\{f^+,e^+,+)\\
	(d^-,c^-,+), (c^-,d^-,+) \end{array}\right\} & \{(b^+,\textit{assume},+)\} & \langle \{e,f\}, \{c,d\}\rangle\\
S(M_5) &  \left\{\begin{array}{c}
	(e^+,\top,+),\{f^+,e^+,+),\\
	(d^-,c^-,+), (c^-,d^-,+), \\
	(a^-,\textit{assume},-), \\
	(b^+,e^+,+), (b^+,a^-,-) \end{array}\right\} & \emptyset  & \langle \{e,f,b\}, \{c,d,a\}\rangle\\
\hline
  \end{array}
\]
\hfill$\Box$
\end{example}

\begin{example}
Let $P$ be the program:
\[\begin{array}{lclclcl}
\texttt{p} & \uffa & \naf \texttt{q} &\hspace{.5cm} &
\texttt{q} & \uffa & \naf \texttt{p} \\
\texttt{r} & \uffa & \naf \texttt{p} &\hspace{.5cm} &
\texttt{p} & \uffa & \texttt{r} \\
  \end{array}
\]
This program does not admit any answer sets where $p$ is false. 
One possible computation (we 
highlight only steps that change the trace):
\[ \begin{array}{lllcclllcclllc}
1. &\langle \emptyset, \emptyset \rangle & \mapsto_{choice}  & \hspace{.5cm} &
	2. & \langle \emptyset,\{p\} \rangle & \mapsto_{AL^1}\\
3. & \langle \{q\}, \{p\} \rangle & \mapsto_{AL^1} &&
4. & \langle \{q,r\},\{p\}\rangle & \mapsto_{AL^1} & \hspace{.5cm} 
5. & \langle \{q,r,p\}, \{p\} \rangle 
   \end{array}
\]
{From} this computation we can obtain a sequence of snapshots:

\[\begin{array}{|r||c|c|c|}
\hline
      &      E_1       &        E_2           &      D\\
\hline
 S(M_0)  &    \emptyset   &    \emptyset         &   \emptyset \\
 S(M_1)  & \{(p^-,\textit{assume},-)\} &  \emptyset  & \langle \emptyset, \{p\}\rangle\:  \rangle\\
 S(M_2)  & \{(p^-,\textit{assume},-), (q^+,p^-,-)\}     &  \emptyset     & \langle\{q\}, \{p\} \rangle\\
 S(M_3)  & \{(p^-,\textit{assume},-), (q^+,p^-,-), (r^+,p^-,-)\} & \emptyset & \langle \{q,r\},\{p\} \rangle\\ 
 S(M_4)  & \left \{
\begin{array}{c}
(p^-,\textit{assume},-), (q^+,p^-,-), \\
(r^+,p^-,-), (p^+,r^+,+)
\end{array}
\right\} & \emptyset & \langle \{p,q,r\},\{p\} \rangle\\
\hline
  \end{array}
\]
Observe that a conflict is detected by the computation and the
sources of conflict are highlighted in the presence of two justifications
for $p$, one for $p^+$ and another one for $p^-$.
\hfill$\Box$
\end{example}

\subsection{Discussion}\label{imple}

In this subsection, we discuss possible ways to extend the notion of justifications 
on various language extensions of ASP. We also describe a system capable of computing 
off-line and on-line justifications for ASP programs.  
 
\subsubsection{Language Extensions}
In the discussion presented above, we relied on a standard logic programming
language. Various systems, such as {\sc Smodels}, have introduced language
extensions, such as choice atoms, to facilitate program development. The extension
of the notion of justification to address these extensions is relatively 
straightforward.

Let us consider, for example, the choice atom construct of {\sc Smodels}. 
A choice atom has the form $L \leq \{a_1,\dots,a_n,not\:b_1,\dots,not\:b_m\}\leq U$
where $L, U$ are integers (with $L\leq U$) and the various $a_i, b_j$ 
are atoms. Choice atoms are allowed to appear both in the head as well
as the body of rules. Given an interpretation $I$ and a choice atom,
we say that $I$ satisfies the atom if
\[ L \leq |\{a_i\:|\: a_i \in I^+\}| + |\{b_j \:|\: b_j \in I^-\}| \leq U\]

The local consistent explanation of a choice atom can be developed in 
a natural way:
\begin{itemize}
\item If the choice atom $L\leq T \leq U$ is true, then
a set of literals $S$ is an LCE if
	\begin{itemize}
	\item ${\cal A}\cap S \subseteq T$ and $not\:{\cal A}\cap S \subseteq T$
	\item  for each $S'$ such that $S\subseteq S'$ and 
		$\{atom(\ell)\:|\: \ell \in S'\}=\{atom(\ell)\:|\: \ell \in T\}$ we
		have that 
\[L\leq |\{a\:|\: a\in T\cap {\cal A}\cap S'\}|+|\{b\:|\:not\:b \in T\cap S'\}|\leq U\]
	\end{itemize}

\item if the choice atom $L\leq T \leq U$ is false, then
	a set of literals $S$ is an LCE if 
	\begin{itemize}
	\item ${\cal A}\cap S \subseteq T$ and $not\:{\cal A}\cap S \subseteq T$

	\item  for each $S'$ such that $S\subseteq S'$ and 
		$\{atom(\ell)\:|\: \ell \in S'\}=\{atom(\ell)\:|\: \ell \in T\}$ we
		have that 
\[\begin{array}{c}
	L > |\{a\:|\: a\in T\cap {\cal A}\cap S'\}|+|\{b\:|\:not\:b \in T\cap S'\}|\\
	\textit{or}\\
	|\{a\:|\: a\in T\cap {\cal A}\cap S'\}|+|\{b\:|\:not\:b \in T\cap S'\}| > U
  \end{array}
\]
	\end{itemize}
\end{itemize}
The notions of e-graphs can be extended to include choice atoms. Choice atoms
in the body are treated as such and justified according to the new notion
of LCE. On the other hand, if we have a rule of the type
\[ L\leq T \leq U \:\uffa\: Body \]
and $M$ is an answer set, then we will 
\begin{itemize}
\item treat the head as a new (non-choice) atom ($new_{L\leq T \leq U}$), 
	and allow its justification
	in the usual manner, using the body of the rule
\item for each atom $p\in T\cap M^+$, the element $p^+$ has a new LCE
	$\{new_{L\leq T \leq U}\}$
\end{itemize}

\begin{example}
Consider the program containing the rules:
\[\begin{array}{lclclcl}
\multicolumn{2}{l}{\texttt{p}\:\:\:\: \uffa} & &\hspace{.5cm} & \texttt{q} & \uffa & \\
2\leq \{\texttt{r,t,s}\}\leq 2 & \uffa & \texttt{p, q} & & 
  \end{array}
\]
The interpretation $\langle \{t,s,p,q\},\{r\}\rangle$ is an answer
set of this program. The off-line justifications for $s^+$ and
$t^+$ are illustrated in 
Figure~\ref{choiceexp}.
\hfill$\Box$
\end{example}
The concept can be easily extended to deal with weight atoms.

\begin{figure}[htb]
\centerline{\psfig{figure=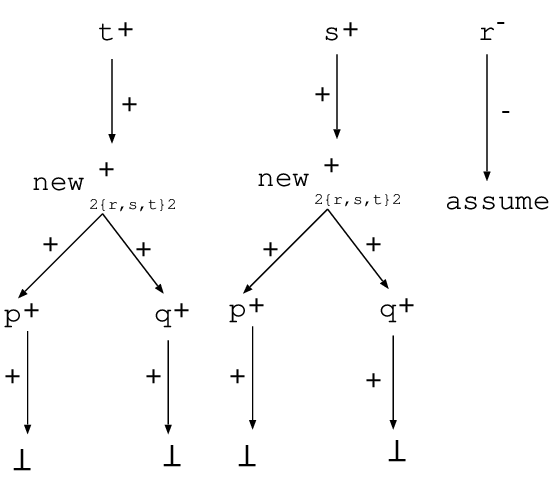,width=.5\textwidth}}
\caption{Justifications in presence of choice atoms}
\label{choiceexp}
\end{figure}

\subsubsection{Concrete Implementation}

The notions of off-line and on-line justifications proposed in the previous sections have been
implemented and integrated in a debugging system for Answer Set Programming,
developed within the \ASP\ framework~\cite{asp-prolog}.
The notions of justification proposed here is meant to represent the
basic data structure on which debugging strategies for ASP can be developed. 
 \ASP\ allows the construction of Prolog
programs---using CIAO Prolog~\cite{GrasH00}---which include 
modules written in ASP (the {\sc Smodels} flavor of ASP). 
In this sense, the embedding of ASP within a Prolog framework (as possible
in \ASP) allows the programmer to use Prolog itself to query the justifications
and develop debugging strategies. We will begin this section with a
short description of the system \ASP.

The \ASP\ system has been developed using the module and class capabilities
of CIAO Prolog. \ASP\ allows programmers to develop programs as
collections of \emph{modules}. Along with the traditional
types of modules supported by CIAO Prolog (e.g., Prolog modules, Constraint
Logic Programming modules), it allows the presence
of  \emph{ASP modules}, each being a complete ASP
program. Each CIAO Prolog module can access the content of
any ASP module (using the traditional module qualification of Prolog),
read its content, access its models, and modify it (using the 
traditional {\tt assert} and {\tt retract} predicates of Prolog).
 
\begin{example}
\ASP\ allows us to create Prolog modules that access (and possibly modify)
other modules containing ASP code. For example, the following Prolog
module
\begin{verbatim}
   :- use_asp(aspmod, 'asp_module.lp').
  
   count_p(X) :- 
        findall(Q, (aspmod:model(Q), Q:p), List),
        length(List,X).
\end{verbatim}
accesses an ASP module (called {\tt aspmod}) and 
defines a predicate ({\tt count\_p}) which counts how many
answer sets of {\tt aspmod} contain the atom {\tt p}.
\qed
\end{example}

\paragraph{Off-Line Justifications:}
The {\sc Smodels} engine has been modified to extract, during the computation,
a compact footprint of the execution, i.e., a trace of the key events (corresponding
to the transitions described in Sect.~\ref{smo}) with links to the atoms and rules
involved. The modifications of the trace are trailed to support backtracking. Parts of
the justification (as described in the previous section) are built on the fly, while
others (e.g., certain cases of $AL^3$ and $AL^4$) are delayed until
the justification is requested.

To avoid imposing the overhead of justification construction on every computation, the
programmer has to specify what ASP modules require justifications, using an
additional argument ({\tt justify}) in the module import declaration:
\[
\texttt{:- use\_asp($\langle$ module\_name $\rangle$, $\langle$ file\_name 
		$\rangle$, $\langle$ parameters $\rangle$ [,justify]).}
\]
Figure~\ref{ovju} shows a general overview of the implementation of ASP justifications
in \ASP. Each program is composed of CIAO Prolog modules and ASP modules (each containing
rules of the form (\ref{rule}), possibly depending on the content of other ASP/Prolog modules). The implementation
of \ASP, as described in~\cite{asp-prolog}, automatically generates, for each ASP module,
an \emph{interface module}---which supplies the predicates to access/modify the ASP module and its
answer sets---and a \emph{model class}---which allows the encoding of each answer set
as a CIAO Prolog object~\cite{Pineda99}. The novelty is the extension of the model class, to 
provide access to the justification of the elements in the corresponding answer set.

\ASP\ provides the predicate {\tt model/1} to retrieve answer sets
of an ASP module---it retrieves them in the order they are computed 
by {\sc Smodels}, and it returns the current
one if the computation is still in progress. The main predicate
to access the justification is {\tt justify/1} which retrieves a CIAO
Prolog object 
containing the justification; i.e.,
\[
\texttt{?- my\_asp:model(Q), Q:justify(J).}
\]
will assign to {\tt J} the object containing the justification relative to the
answer set {\tt Q} of the ASP module {\tt my\_asp}.
Each justification object provides the following predicates: 
\begin{itemize}
\item {\tt just\_node/1} which succeeds if
the argument is one of the nodes in the justification graph, 
\item {\tt just\_edge/3} which succeeds if
the arguments correspond to the components of one of the edges in the graph, and
\item {\tt justify\_draw/1} which will generate a graphical drawing of the 
	justification for the given
	atom (using the \emph{uDrawGraph} application). An example display
	produced by \ASP\ is shown in Figure~\ref{gf1}; observe that rule names
	are also displayed to clarify the connection between edges of a 
	justification and the generating program rules. 
\end{itemize}
For example,
\[
\texttt{ ?- my\_asp:model(Q),Q:justify(J),findall(e(X,Y),J:just\_edge(p,X,Y),L).}
\]

\noindent
will collect in {\tt L} all the edges supporting {\tt p} in the
justification graph (for answer set {\tt Q}).

\begin{figure}[htbp]
\centerline{\psfig{figure=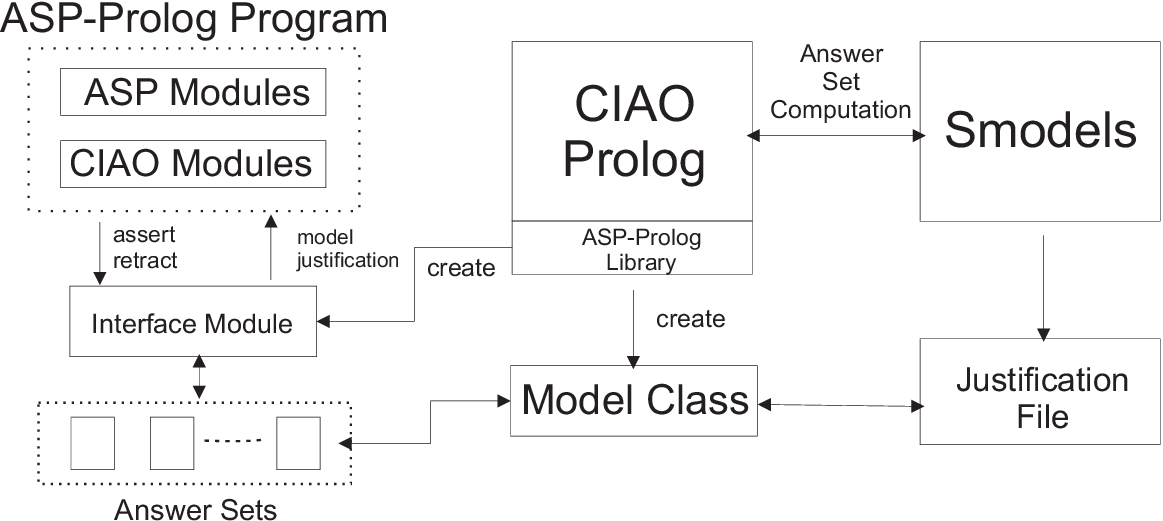,width=.7\textwidth}}
\caption{\ASP\ with justifications}
\label{ovju}
\end{figure}

\begin{figure}[htbp]
\centerline{\psfig{figure=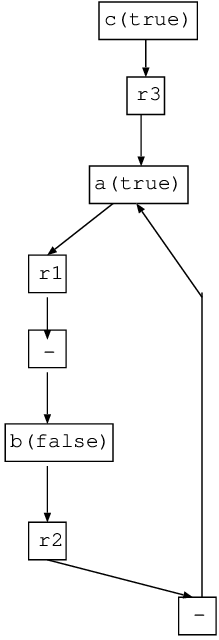,width=.2\textwidth}}
\caption{An off-line justification produced by \ASP}
\label{gf1}
\end{figure}

\paragraph{On-Line Justifications:}
The description of {\sc Smodels} on-line justifications we proposed earlier is
clearly more abstract than the concrete implementation---e.g., we did not
address the use of lookahead, the use of heuristics, and other optimizations
introduced in {\sc Smodels}.  All these elements
have been  handled in the current implementation, in the same spirit of what described here.

On-line justifications have been integrated in the \ASP\ system as
part of its  ASP debugging facilities.
The system  provides predicates to set breakpoints on the execution of an
ASP module, triggered by events such as the assignment of a truth value to a certain atom
or the creation of  a conflicting assignment. Once a breakpoint is encountered,
it is possible to visualize the current on-line justification or step through
the rest of the execution.
Off-line justifications 
are always available. 

The {\sc Smodels} solver  is in charge of handling the activities
of interrupting and resuming execution, during the computation of an answer set
of an ASP program. A synchronous communication is maintained between a Prolog
module and an ASP module---where the Prolog module requests and controls
the ASP execution. When the ASP solver breaks, e.g., because a breakpoint is 
encountered, it sends a compact encoding of its internal data structures to the
Prolog module, which stores it in a ASP-solver-state table. If the Prolog module
requests resumption of the ASP execution, it will send back to the solver
the desired internal state, that will allow continuation of the execution. This
allows the execution to be restarted from any of a number of desired points
(e.g., allowing a ``replay''-style of debugging) and to control different ASP
modules at the same time.

\ASP\ provides the ability to establish a number of different types of breakpoints
on the execution of an ASP module. In particular, 
\begin{itemize}
\item {\tt break(atom,value)} interrupts the execution when the {\tt atom} is assigned
	the given value; {\tt value} could {\tt true}, {\tt false} or {\tt any}.
\item {\tt break(conflict)} interrupts the execution whenever a conflict is 
	encountered during answer set computation.\footnote{Here, we refer to 
	conflict in the same terms as {\sc Smodels}.}
\item {\tt break(conflict(atom))} interrupts the execution if a conflict involving
	the  {\tt atom} is encountered.
\item {\tt break(answer(N))} interrupts the execution at the end of the computation 
	of the answer set referred to by the object {\tt N}.
\end{itemize}
Execution can be restarted using the built-in predicate {\tt run}; the partial results
of an interrupted computation (e.g., the partial answer set, the on-line justification)
can be accessed using the predicates {\tt model} and {\tt justify}.

\begin{example}
Consider the following fragment of a Prolog program:
\begin{verbatim}
    :- module ( p, [m/0] ).
    :- use_asp ( asp, 'myasp.lp', justify ).

    m :- asp:break(atom(a,true)),
         asp:run,
         asp:model(Q),
         Q:justify(J),
         J:justify_draw(a).
\end{verbatim}
This will stop the execution of the answer set program {\tt myasp.lp} whenever
the atom {\tt a} is made true; at that point, the Prolog program shows a 
graphical representation of the corresponding on-line justification of {\tt a}.
\hfill$\Box$
\end{example}

\subsection{Justifications and Possible Applications} 

The previous subsection discusses a possible application of the notion 
of justification developed in this paper, namely the construction of 
an interactive  debugging system for logic programs under the 
answer set semantics.
It is worth mentioning that the notion of justification is 
general and can be employed in other applications as well.  
We will now briefly discuss other potential uses of this concept. 

Thanks to their ability to explain the presence and absence of
atoms in an answer set, off-line justifications provide a 
natural solution to problems in the domain of  ASP-based diagnosis.
As in systems like
\cite{BalducciniG03}, off-line justifications 
can help in discriminating 
diagnoses. 
Let us consider, for example, a system composed of two components, 
$c_1$ and $c_2$. Let us assume that there is a  dependence between these
components,  stating that 
if $c_1$ is defective then $c_2$ will be defective as well. 
This information can be expressed by the following rule:  
\[
h(ab(c_2),T) \:\: \uffa \:\: h(ab(c_1),T)  
\]
where $h(ab(c_1), t)$ (resp. $h(ab(c_2), t)$) being true
indicates that the component $c_1$ (resp. $c_2$) is defective
at an arbitrary time $T$.  

Given this rule, $h(ab(c_2), t)$ ($ab(c_2)$ is defective) 
belongs to any answer set which contain $h(ab(c_1), t)$ 
($ab(c_1)$ is defective). Thus, any off-line justification for 
$h(ab(c_1), t)^+$ can be extended to an off-line justification for 
$h(ab(c_2), t)^+$ by adding a positive edge from 
$h(ab(c_2), t)^+$ to $h(ab(c_1), t)^+$. This is another argument,
besides the minimality criteria, for preferring 
the diagnosis $\{c_1\}$ over $\{c_1,c_2\}$.  

The implemented system for on-line justification in this paper
can be adapted to create a direct implementation of the 
CR-Prolog \cite{Balduccini07}. Currently, a generate-and-test 
front-end to {\sc Smodels} is provided for computing 
answer sets of CR-Prolog programs. More precisely,
the algorithm for computing the answer sets of a CR-Prolog
program $P$, whose set of normal rules is $Q$,
iterates through two steps until an answer set 
is found. In each iteration,  a minimal set of CR-rules 
is selected randomly (or according to some preferences), 
activated (i.e., converted to normal rules) 
and added to $Q$ to create a new program $Q'$. 
The answer sets of $Q'$ are computed using {\sc Smodels}.
If any answer set is found, then the computation stops.

This implementation does not make use of any information 
about possible conflicts or inconsistencies that can be 
recognized during the computation. 
A more effective implementation can be achieved by 
collecting on-line justifications during each cycle of
execution of {\sc Smodels}. The on-line justifications can
be traversed to identify inconsistencies and identify
rules outside of $Q$ that unavoidably conflict with 
rules in $Q$. Such knowledge can then be employed to
suggest more effective selections of CR-rules to
be activated.

\begin{example}
Consider the following simple CR-Prolog program
\[\begin{array}{cclcl}
r_1 & \:\:& a & \uffa & not\:b.\\
r_2&&\neg a \\
r_3&&b & \stackrel{+}{\leftarrow}\\
r_4&&c & \stackrel{+}{\leftarrow}
  \end{array}
\]
In this case, the set of normal rules $Q$ contains the 
two rules $r_1$ and $r_2$, and $Q$ does not admit a (consistent)
answer set. The point of conflict is characterized by the on-line
justification shown in Figure~\ref{cr1}. The conflict is clearly 
represented by the presence of justifications for $a^+$ and
$(\neg a)^+$; the justification also highlights that the only
way of defeating the conflict is to remove the positive edge
between $not\:b$ and $a^+$. This suggests the need of introducing
a CR-rule that has $b$ as head, i.e., rule $r_3$.

Simple \ASP\ meta-interpreters can be introduced to detect this
type of situations and suggest some consistency restoring rules
to be used; e.g., given the partial answer set $M$ 
present at the time of conflict, we can use the following clause to
resolve conflicts due to atoms of the type $p$ and $\neg p$ both
being true:
\[\begin{array}{l}
 \texttt{candidate\_rule}(Y\stackrel{+}{\leftarrow}Body,M)  \uffa\\
\hspace{1cm}     M:\texttt{justify}(J),\\
\hspace{1cm}     M:Atom, M:(-Atom),\\
\hspace{1cm}     (  \texttt{reachable}(Atom,Y,J) ; \texttt{reachable}(-Atom,Y,J) ),\\
\hspace{1cm}     M:not\:\: Y,\\ 
\hspace{1cm}     \stackrel{+}{\leftarrow}(Y,Body).
  \end{array}
\]
where {\tt reachable} performs a simple transitive closure over
the edges of the justification $J$.

\begin{figure}[htbp]
\centerline{\psfig{figure=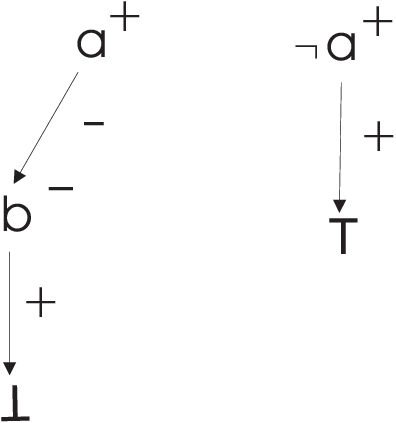,width=.2\textwidth}}
\caption{On-line justification for the rules $r_1$ and $r_2$}
\label{cr1}
\end{figure}
\hfill$\Box$

\end{example}

\section{Related Work} 

Various approaches to logic program understanding and debugging 
have been investigated (and a thorough comparison is beyond the
limited space of this paper). Early work in this direction geared 
towards the understanding of Prolog programs rather than logic 
programs under the answer set semantics. Only recently, we can 
find some work on debugging inconsistent programs or providing explanation
for the presence (or absence) of an atom in an answer set. 
While our notion of justification is related to the research 
aimed at debugging Prolog and XSB programs, its initial implementation 
is related to the recent attempts in debugging logic programs 
under the answer set semantics. We will discuss each of these issues
in each subsection. 

\subsection{Justifications and Debugging of Prolog Programs} 

As discussed in \cite{PemmasaniGDRR04}, $3$ main phases can be considered
in understanding/debugging a logic program:

\begin{itemize}
\item[1.] {\em Program instrumentation and execution:}
	assertion-based
	debugging (e.g., \cite{PueblaBH98}) and  
	algorithmic debugging \cite{Shapiro82} are examples of approaches focused on this first phase.

\noindent
\item[2.] {\em Data Collection:}
 focuses on \emph{extracting} from the execution data
	necessary to understand it, as in  event-based debugging
	\cite{Auguston00}, tracing, and explanation-based debugging
	\cite{Ducasse99,MalletD99}.

\noindent
\item[3.] {\em Data Analysis:} focuses on reasoning on data collected during
	the execution. The proposals dealing with automated debugging (e.g., \cite{Auguston00})
	and execution visualization
	(e.g., \cite{VaupelGP97}) are approaches focusing on this
	phase of program understanding.
\end{itemize}

\noindent
The notion of \emph{Justification} has been introduced in
\cite{PemmasaniGDRR04,RoychoudhuryRR00,Specht93} to support
understanding and debugging of logic programs.
Justification is the process 
of generating evidence, in terms of high-level proofs based on the answers (or models) 
produced during the computation. Justifications are \emph{focused}, i.e., 
they provide only the information that are relevant to the item being
explained---and this separates them from other debugging schemes (e.g., tracing).
Justification plays an important role 
in manual and automatic verification, by providing a \emph{proof description}
if a given property holds; otherwise, it
generates a \emph{counter-example}, showing where the violation/conflict
occurs in the system. 
The justification-based approach focuses on the last two phases of
debugging---collecting
data from the execution and presenting them in a meaningful manner. Differently
from generic tracing and algorithmic debugging, justifications are focused only
on parts of the computation relevant to the justified item. Justifications
are fully automated and do not require user interaction (as in declarative
debugging). 

Justifications relies on describing the evidence for an answer in 
terms of a graph structure. The term \emph{justification} was introduced
in \cite{RoychoudhuryRR00}, as a data structure to explain answers to 
Prolog queries within a Prolog system with tabling. The notion of justification
and its implementation in the XSB Prolog system was successively refined
in~\cite{PemmasaniGDRR04,Guo01}. Similar structures have been suggested to 
address the needs of other flavors of logic programming---e.g., various approaches
to tree-based explanation for deductive databases (e.g., the \emph{Explain}
system~\cite{explain}, the explanation system for LOLA~\cite{Specht93}, and the
DDB trees method~\cite{MalletD99}). Similar methods have also been developed
for the analysis of CLP programs (e.g.,~\cite{clpdeb}).

In this work, we rely on graph structures as a mean to 
describe the \emph{justifications} that are generated during the generation 
(or from) an answer set of a program. Graphs have been
used in the context of logic programming for a variety of other applications.
\emph{Call graphs} and \emph{dependence graphs} have been extensively
used to profile and discover program properties (e.g.,~\cite{call1,call2}).
\emph{Support graphs} are used for program analysis in \cite{SahaR05}.

The use of graphs proposed in this paper 
is complementary to the view proposed by other authors, who  use graph
structures as a mean to describe answer set programs, to make structural
properties explicit, and to support
the development of the program execution. In~\cite{AngerGLNS05,color1},  
\emph{rule dependency graphs} (a.k.a. \emph{block graphs}) of answer
set programs are employed to model the computation of answer sets
as special forms of graph coloring. A comprehensive survey of alternative graph
representations of answer set programs, and their properties with respect
to the problem of answer set characterization, has been presented
in~\cite{constantini1,CostantiniDP02}. In particular, the authors provide
characterizations of desirable graph representations, relating the 
existence of answer sets to the presence of cycles and the use of coloring
to characterize properties of programs (e.g., consistency). 
We conjuncture that 
the outcome of a successful coloring of an EDG~\cite{constantini1} to
represent one answer set can be projected, modulo non-obvious transformations,
to an off-line graph and vice versa. On the other hand, the notion of on-line
justification does not seem to have a direct relation to the graph representations
presented in the cited works.

\subsection{Debugging Logic Programs under Answer Set Semantics} 

This paper continues the work initiated in~\cite{ElKhatibPS05}, by
proposing a more advanced and sound notion of off-line justification, by developing the concept
of on-line justification, and introducing these concepts
in {\sc Smodels}. The approach differs significantly from 
the recently introduced approach to debugging ASP programs in \cite{BrainGP+b07}
While our approach relies on the notion of justification, 
the approach in \cite{BrainGP+b07} uses the tagging technique 
\cite{DelgrandeST03} to compile a program into a new program whose 
answer sets can be used to debug the original program. 
Inspecting an answer set of the new program can reveal 
the rules which have been applied in its generation. It does not,
however, provides explanation of why an atom does (or does not) 
belong to the answer set. In this sense, we can say that the approach 
of \cite{BrainGP+b07} and ours are complementary to each other. 
An advantage of the approach in \cite{BrainGP+b07} is that it enables 
the development of a debugger as a front-end of an answer set solver. 
However, their approach does not consider on-line justification. 

At this point, it is worth mentioning that the \ASP\ debugger,
described in Section \ref{smo}, differs from the system {\tt 
spock} \cite{BrainGP+a07}---which was developed based on the technical 
foundation in \cite{BrainGP+b07}---in several aspects. 
In our system, the justification for the truth value of an 
atom consists of facts, assumptions, and rules which are applicable 
given these facts and assumptions, i.e., we not only justify why an atom 
is {\em true} but also why an atom is {\em false}. Moreover, 
justifications can be queried during the process of answer set
computation. {\tt spock} only provides the justification,
or the applicable rules, for the presence of an atom in a 
given answer set. In this sense, justifications in {\tt spock}
is similar to our off-line LCEs. 

In \cite{dlvdeb}, 
a tool for developing and testing DLV programs was described. 
The commands provided by this tool allow an user to 
inspect why an atom is true in the current model and 
why there is no answer set. This is similar to the on-line justifications
developed for {\sc Smodels}. The tool in \cite{dlvdeb}, however, 
does not answer the question 
why an atom is not in the current model. The notion of justifications 
is not developed in \cite{dlvdeb}. 

The proposed debugger is similar to the system described in \cite{BrainDv05}
in that it provides the users with the information on why some 
atoms occur in an answer set and some others do not. 
An explanation given by the tool described in this work is
similar to an off-line justification in our work. Our 
implementation also provides users with on-line justifications 
but the system described in \cite{BrainDv05} does not. 

The paper \cite{Syrjanen06} presents a theory for debugging of inconsistent
programs and an implementation of this theory. The focus of this 
paper is on inconsistent programs. On the other hand, our 
focus is not solely on inconsistent programs. Our 
notion of on-line justification can be used in identifying the 
reasons that lead to the inconsistency of the problem but
it is significant different from the theory of diagnosis 
developed in \cite{Syrjanen06}.

\section{Conclusion}
In this paper we provided a generalization of the notion of
\emph{justification} (originally designed for Prolog with
SLG-resolution \cite{RoychoudhuryRR00}), to suit the needs of ASP.
The notion, named \emph{off-line
justification}, offers a way to understand the motivations for the
truth value of an atom within a specific answer set, thus making it
easy to analyze answer sets for program understanding and debugging. 
We also introduced \emph{on-line justifications}, which are
meant to justify atoms \emph{during} the computation of an answer set. The
structure of an on-line justification is tied to the specific steps
performed by a computational model for ASP (specifically,
the computation model adopted by 
{\sc Smodels}). An on-line justification allows a programmer to inspect
the reasons for the truth value of an atom at the moment such value
is determined while constructing an answer set.
These data
structures provide a foundation for the construction of tools to understand
and debug ASP programs.

The process of
computing and presenting justifications has been embedded in
the ASP-Prolog system~\cite{asp-prolog}, thus making justifications a first-class
citizen of the language. This allows the programmer to use Prolog to 
manipulate justifications as standard Prolog terms. A prototype implementation 
has been completed and is currently under testing. 

As future work, we propose to complete the implementation, refine the definition
of on-line justification to better take advantage of the {\sc Smodels} mechanisms,
and develop a complete debugging and visualization environment for ASP based on these
data structures.

\medskip \noindent 
{\bf Acknowledgement}: We would like to thank the anonymous reviewers 
for their comments and suggestions that help improve the papers 
in many ways. The authors are 
partially supported by NSF grants CNS-0220590, HRD-0420407, and 
IIS-0812267.

\bibliographystyle{acmtrans}

\section*{Appendix: Proofs}

\subsection*{Proof of Proposition \ref{prop1}.}

\noindent
{\em Proposition \ref{prop1}.}
{\it Given a program $P$ and an answer set $M$ of $P$, the well-founded model of 
$NR(P,{\cal TA}_P(M))$ is equal to $M$. }

\smallskip

The following result has been proved~\cite{apt94a}.
\begin{theorem}
Let $P$ be a program and $j$ be the first index such that
$(K_j,U_j) = (K_{j+1},U_{j+1})$. The well-founded model of $P$,
$WF_P=\langle W^+,W^- \rangle$, satisfies 
$W^+ = K_j$ and $W^- = {\cal A} \setminus U_j$. 
\end{theorem}

Let $T_R$ denote the traditional immediate consequence
operator for a definite program $R$ \cite{llo87}. We will also 
make use of the usual notations such as  $T_R \uparrow 0 = \emptyset$, 
$T_R \uparrow i = T_R(T_R \uparrow (i-1))$.  
Given a program $P$ and one of its
answer sets $M$, to simplify the presentation, 
let us denote with $Q(M)$ the
negative reduct $NR(P, {\cal TA}_P(M))$.
We will denote with $(K^P_i,U^P_i)$ 
the pair $(K_i,U_i)$ (Definition \ref{kui}) 
for the original program $P$ and with $(K^Q_i,U^Q_i)$ 
the pair $(K_i,U_i)$ for program $Q(M)$ respectively.

\begin{lemma} \label{le-a1}
For a program $P$, $lfp(T_{P^+}) = lfp(T_{Q(M)^+})$.
\end{lemma}
\noindent {\bf Proof.}
Clearly, $lfp(T_{P^+}) \supseteq lfp(T_{Q(M)^+})$ since 
$P^+ \supseteq Q(M)^+$. 

Let us prove, by induction on $i$, that 
$T_{P^+} \uparrow i \subseteq T_{Q(M)^+} \uparrow i$.
The result is trivial for the base case. 
Let us assume that the result holds for $i$ and
let us  prove it for $i+1$.
Consider $a \in T_{P^+} \uparrow i+1$. This means that 
there is a rule $r \in P^+$ such that $head(r) = a$ and 
$pos(r) \subseteq T_{P^+} \uparrow i$. 
Since $a \in lfp(T_{P^+})\subseteq M^+$, we know 
that $a \in M^+$ and therefore $r \in Q(M)^+$. Thus, 
thanks to the inductive hypothesis, we can conclude that 
$a \in T_{Q(M)^+} \uparrow i+1$. 
\hfill $\Box$ 

\begin{corollary}\label{co1}
For a program $P$, $K^P_0 = K^Q_0$. 
\end{corollary}

\begin{lemma}\label{l2}
For a program $P$, $U^Q_0 \subseteq U^P_0$. 
\end{lemma}
\noindent {\bf Proof.}
We prove, by induction on $i$, that 
$T_{Q(M),K^Q_0} \uparrow i \subseteq T_{P,K^P_0} \uparrow i$.

\noindent 
{\bf Base:} The result is obvious for $i=0$, since
\[T_{Q(M),K^Q_0} \uparrow 0 = \emptyset = T_{P,K^P_0} \uparrow 0\]
Let $a \in T_{Q(M),K^Q_0} \uparrow 1$. 
This implies that there is  $r \in Q(M)$ such that $head(r) = a$,
$pos(r) = \emptyset$, and $neg(r) \cap K^Q_0 =\emptyset$.
Since $Q(M) \subseteq P$, we also have that
$r \in P$. Furthermore, since 
$K^P_0 = K^Q_0$ (from Corollary~\ref{co1}), 
we have that $a \in T_{P,K^P_0} \uparrow 1$.

\noindent 
{\bf Step:} Let us assume the result to be true for $i$ and
let us consider the iteration $i+1$.
Let $a \in T_{Q(M),K^Q_0} \uparrow i+1$. 
This implies that there is a rule $r \in Q(M)$ such that
\begin{itemize}
\item  $head(r) = a$,
\item $pos(r) \subseteq T_{Q(M),K^Q_0} \uparrow i$, and 
\item $neg(r) \cap K^Q_0 =\emptyset$.
\end{itemize}
Since $Q(M) \subseteq P$, then we have that  $r \in P$. Furthermore, since 
$K^P_0 = K^Q_0$, we have that $a \in T_{P,K^P_0} \uparrow i+1$.
\hfill $\Box$ 

\begin{proposition}
\label{sequence}
For every $i$, $U^Q_i \subseteq U^P_i$ and $K^P_i \subseteq K^Q_i$. 
\end{proposition}
\noindent {\bf Proof.}
We will  prove this result by induction on $i$.
The base case follows directly from Lemmas \ref{le-a1}-\ref{l2}. 
Let us proceed with the inductive step.

First, let us prove by induction on $j$ that 
$T_{P,U^P_{i-1}} \uparrow j \subseteq T_{Q(M),U^Q_{i-1}} \uparrow j$.

\begin{itemize}

\item {\bf Base:} 
Let $a \in T_{P,U^P_{i-1}} \uparrow 1$. 
This implies that there is a rule $r \in P$ such that $head(r) = a$,
$pos(r) = \emptyset$, and $neg(r) \cap U^P_{i-1} =\emptyset$.
Since $K^P_{i} \subseteq W^+$, we have that 
$a \not\in {\cal TA}(M)$, and thus $r \in Q(M)$. Furthermore, since 
$U^Q_{i-1} \subseteq U^P_{i-1}$,
we have that 
$neg(r) \cap U^Q_{i-1} = \emptyset$. Hence, 
$a \in T_{Q(M),U^Q_{i-1}} \uparrow 1$.

\item {\bf Step:} let us assume the result to hold for
$j$ and let us prove it for $j+1$.
Let $a \in T_{P,U^P_{i-1}} \uparrow j+1$. 
This implies that there is a rule $r \in P$ such that 
\begin{itemize}
\item $head(r) = a$,
\item $pos(r) \subseteq T_{P,U^P_{i-1}} \uparrow j$, and 
\item $neg(r) \cap U^P_{i-1} =\emptyset$.
\end{itemize}
Since $K^P_{i} \subseteq W^+$, we have that 
$a \not\in {\cal TA}(M)$, and thus $r \in Q(M)$. Furthermore, since 
$U^Q_{i-1} \subseteq U^P_{i-1}$, we have that 
$neg(r) \cap U^Q_{i-1} = \emptyset$. 
By inductive hypothesis, we have that 
$pos(r) \subseteq T_{Q(M),U^Q_{i-1}} \uparrow j$.
Hence, $a \in T_{Q(M),U^Q_{i-1}} \uparrow j+1$.
\end{itemize}

\medskip \noindent
Let us now prove, by induction on $j$ that 
$T_{Q(M),K^Q_{i}} \uparrow j \subseteq T_{P,K^P_{i}} \uparrow j$.

\begin{itemize}
\item {\bf Base:} 
Let $a \in T_{Q(M),K^Q_{i}} \uparrow 1$. 
This implies that there is a rule  $r \in Q(M)$ such that $head(r) = a$,
$pos(r) = \emptyset$, and $neg(r) \cap K^Q_{i} =\emptyset$.
Since $Q(M) \subseteq P$, we have that $r \in P$. Furthermore, since 
$K^P_{i} \subseteq K^Q_{i}$, we have that 
$neg(r) \cap K^P_{i} = \emptyset$. Hence, 
$a \in T_{P,K^P_{i}} \uparrow 1$.

\item {\bf Step:} let us assume the result to hold for $j$ and
let us consider the case $j+1$.
Let $a \in T_{Q(M),K^Q_{i}} \uparrow j+1$. 
This implies that there is a rule $r \in Q(M)$ such that $head(r) = a$,
$pos(r)\subseteq T_{Q(M),K^Q_{i}} \uparrow j$, 
and $neg(r) \cap K^Q_{i} =\emptyset$.
Since $Q(M) \subseteq P$, we have that $r \in P$. Furthermore, since 
$K^P_{i} \subseteq K^Q_{i}$, we have that 
$neg(r) \cap K^P_{i} = \emptyset$. 
By inductive hypothesis, we also have that 
$pos(r) \subseteq T_{P,K^P_{i}} \uparrow j$. 
Hence, 
$a \in T_{P,K^P_{i}} \uparrow j+1$.
\end{itemize}
\hfill$\Box$

\begin{lemma}
If $M$ is an answer
set  of $P$, then $M$ is an answer set of $Q(M)$.
\end{lemma}
\noindent {\bf Proof.} 
Obviously, $lfp(T_{Q(M)^{M^+}}) \subseteq lfp(T_{P^{M^+}})$ 
because $Q(M)^{M^+} \subseteq P^{M^+}$. Thus, it is sufficient to show that
$lfp(T_{P^{M^+}}) \subseteq lfp(T_{Q(M)^{M^+}})$. 
We prove by induction on $i$ that 
$T_{P^{M^+}} \uparrow i \subseteq T_{Q(M)^{M^+}} \uparrow i$. 

\noindent 
{\bf Base:} Let $a \in T_{P^{M^+}} \uparrow 1$. 
This implies that there is  a rule $r \in P$ such that $head(r) = a$,
$pos(r) = \emptyset$, and $neg(r) \subseteq M^-$. 
Because $a \in M^+$, we have that $r \in Q(M)$. 
Thus, $a \in T_{Q(M)^{M^+}} \uparrow 1$. 

\noindent 
{\bf Step:} Let $a \in T_{P^{M^+}} \uparrow i+1$. 
This implies that there exists a rule $r \in P$ such that $head(r) = a$,
$pos(r) \subseteq T_{P^{M^+}} \uparrow i$, and $neg(r) \subseteq M^-$. 
Since $a \in M^+$, we have that $r \in Q(M)$. 
Thus, $a \in T_{Q(M)^{M^+}} \uparrow i+1$. 
\hfill $\Box$

\medskip
Let us indicate in the rest of this discussion the
well-founded model of $Q(M)$ with $WF_Q$ and the well-founded
model of $P$ with $WF_P$.

\begin{lemma}
${\cal TA}_P(M) \subseteq WF_Q^-$.
\end{lemma}
\noindent 
{\bf Proof.} Consider $a \in {\cal TA}_P(M)$. We have that 
$a \not\in U^Q_i$ for every $i$ 
since there are no rules with $a$ as head in $Q(M)$.
This means that $a \in WF_Q^-$. Thus, 
${\cal TA}_P(M) \subseteq WF_Q^-$.
\hfill $\Box$

\begin{proposition}
The well-founded model $WF_Q$ of $Q(M)$ is equal to $M$, i.e., $W_Q = M$.
\end{proposition}
\noindent 
{\bf Proof.} From Proposition \ref{sequence}, we have that 
$WF_P^+ \subseteq WF_Q^+$ and $WF_P^- \subseteq WF_Q^-$. Furthermore, 
since ${\cal TA}_P(M) \subseteq WF_Q^-$, we can conclude that 
$M^- \subseteq WF_Q^-$. 
Since $M$ is an answer set of $Q(M)$, we also have that 
$WF_Q^- \subseteq M^-$. Thus, $M^- = WF_Q^-$. 
This conclusion implies that there is a value $k$ such that 
$U^Q_k = {\cal A} \setminus M^-$.

Let us now show that $K^Q_{k+1} = M^+$.
Since $M$ is an answer set of $Q(M)$, we immediately
have that 
$K^Q_{k+1} \subseteq M^+$.
Let us prove, by induction on $i$, that 
$T_{P^{M^+}} \uparrow i \subseteq T_{Q(M),U^Q_k} \uparrow i$. 

\noindent 
{\bf Base:} Let $a \in T_{P^{M^+}} \uparrow 1$. 
This implies that there is a rule $r \in P$ 
such that $head(r) = a$,
$pos(r) = \emptyset$, and $neg(r) \subseteq M^-$. 
Since $a \in M^+$, we have that $r \in Q(M)$. 
Furthermore, since $U^Q_k = {\cal A} \setminus M^-$ and 
$neg(r) \subseteq M^-$,
we have that $neg(r) \cap U^Q_k = \emptyset$.
Thus, $a \in T_{Q(M),U^Q_k} \uparrow 1$.

\noindent 
{\bf Step:} Let $a \in T_{P^{M^+}} \uparrow i+1$. 
This implies that there is a rule $r \in P$ such that $head(r) = a$,
$pos(r) \subseteq T_{P^{M^+}} \uparrow i$, and $neg(r) \subseteq M^-$. 
Since $a \in M^+$, we have that $r \in Q(M)$. 
Furthermore, since $U^Q_k = {\cal A} \setminus M^-$ 
and $neg(r) \subseteq M^-$,
we have that $neg(r) \cap U^Q_k = \emptyset$.
By inductive hypothesis, we also have that 
$pos(r) \subseteq T_{Q(M),U^Q_k} \uparrow i$. 
Thus, $a \in T_{Q(M),U^Q_k} \uparrow i+1$. 
\hfill $\Box$

\newpage

\subsection*{Proof of Lemma \ref{good}.}

The proof of this lemma makes use of several results 
and definitions in \cite{BrassDFZ01}.
For this reason, 
let us recall the necessary definitions from \cite{BrassDFZ01}. Given a program
$P$, let us denote with $heads(P) = \{a\:|\: \exists r\in P.\: head(r)=a\}$
and with $facts(P) = \{a \:|\: (a\uffa) \in P\}$.
We can consider the following program transformations~\cite{BrassDFZ01}:
\begin{itemize}
\item $P_1\mapsto_P P_2$ iff $a\uffa body \in P_1$, 
	$not\:b \in body$, $b \not\in heads(P_1)$, and\\
	$P_2 = (P_1 \setminus \{a \uffa body\})\cup\{a \uffa body\setminus\{not\:b\}\}$

\item $P_1 \mapsto_N P_2$ iff $a\uffa body \in P_1$, 
	$not\:b \in body$, $b\in facts(P_1)$, and\\
	$P_2 = P_1 \setminus \{a\uffa body\}$

\item $P_1 \mapsto_S P_2$ iff $a\uffa body \in P_1$,
	$b\in body$, $b \in facts(P_1)$, and \\
	$P_2 = (P_1 \setminus \{a\uffa body\})\cup  \{a \uffa (body\setminus\{b\})\}$

\item $P_1 \mapsto_F P_2$ iff $a\uffa body \in P_1$,
	$b\in body$, $b\not\in heads(P_1)$, and\\
	$P_2 = P_1 \setminus \{a\uffa body\}$

\item $P_1 \mapsto_L P_2$ iff there is a non-empty set of atoms $S$ such that
	\begin{itemize}
	\item for each rule $a\uffa body$ in $P_1$ where $a\in S$ we have that 
		$S\cap body \neq \emptyset$
	\item $P_2 = \{r\in P_1\:|\: body(r) \cap S = \emptyset\}$
	\item $P_1 \neq P_2$
	\end{itemize}
\end{itemize}
We write $P_1 \mapsto P_2$ to indicate that there exists a
transformation $t \in \{P, N, S, F, L\}$ such that 
$P_1 \mapsto_t P_2$. A program $P$ is {\em irreducible} 
if $P \mapsto_t P$ for every $t\in \{P, N, S, F, L\}$.
The results in \cite{BrassDFZ01} show that the above  transformation system
is terminating and confluent, i.e., given a program $P$, 
(a) there exists a sequence of programs $P=P_0,P_1,\ldots,P_n=P^*$ 
such that $P_i \mapsto P_{i+1}$ for $0 \le i \le n-1$
and $P^*$ is irreducible; and 
(b) for every sequence of programs $P=Q_0,Q_1,\ldots,Q_m=Q^*$ 
such that $Q_i \mapsto Q_{i+1}$ for $0 \le i \le m-1$
and $Q^*$ is irreducible then $P^* = Q^*$. We call the 
irreducible program $P^*$ obtained from $P$ through 
this transformation system the normal form of $P$. 
The result in \cite{BrassDFZ01} shows that the well-founded 
model $WF_P=\langle W^+,W^-\rangle$
of $P$ can be obtained by 
\[\begin{array}{lcr}
W^+=facts(P^*) & \hspace{1cm} & W^- = \{a\:|\: a\not\in heads(P^*)\}\end{array}\]
where $P^*$ is the normal form of $P$. 

\medskip
\noindent
{\em Lemma \ref{good}.} 
Let $P$ be a program, $M$ an answer set, and $WF_P$ the well-founded model
of $P$. 
Each atom $a\in WF_P$ has an off-line justification 
w.r.t. $M$ and $\emptyset$ which does not contain any negative cycle.

\medskip
\noindent
{\bf Proof:} 
Let us consider the sequence of transformations of the program 
\[P=P_0 \mapsto P_1 \mapsto \dots \mapsto  P^*\]
such that the transformation $\mapsto_L$ is used only when no other 
transformation can be applied. 
Furthermore, let
\[
WP_i = \langle W_i^+,W_i^-\rangle = \langle facts(P_i), \{a\:|\:a\not\in heads(P_i)\}\rangle
\]
We wish to prove, by induction on $i$, that if $a\in W_i^+\cup W_i^-$ then it
has a justification which is free of negative cycles and it contains 
exclusively elements in $W_i^+\cup W_i^-$. For the sake of simplicity, we will
describe justification graphs simply as set of edges. Also, we will denote
with ${\cal J}(a)$ the graph created for the element $a$.

\noindent
{\bf Base:} Let us consider $i = 0$. We have two cases:

\begin{itemize}
\item $a\in W_0^+$. This means that $a \in facts(P_0) = facts(P)$. 
This implies that ${\cal J}(a)=\{(a^+,\top,+)\}$ is a cycle-free 
justification for $a$ w.r.t. $WP_0$ and $\emptyset$. 

\item $a\in W_0^-$. This means 
that $a \not\in heads(P_0)=heads(P)$. From the definition of off-line justification,
this means that we can build the justification ${\cal J}(a)=\{(a^-,\bot,+)\}$, which is also
	cycle-free. In addition, the only atoms in the justification belongs to
	$W_0^+\cup W_0^-$.
\end{itemize}

\noindent
{\bf Step:} Let us assume that the inductive hypothesis holds for $j \le i$. 
Let us consider $a \in W_{i+1}^+ \cup  W_{i+1}^-$.
We have two cases: 

\begin{itemize}

\item $a \in W_{i+1}^+$. Without loss of generality, we can assume that 
	$a \not\in W_i^+$. 
	This means that the reduction step
	taken to construct $P_{i+1}$ from $P_i$ produced a fact of the form
	$a\uffa$.
	This implies that there exists a rule 
	\[ a \uffa b_1, \dots, b_k, not\:c_1, \dots, not\:c_h \]
	in $P$ such that each $b_j$ has been removed in the previous steps by $\mapsto_S$
	transformations,  and each $not\:c_r$ has been removed by $\mapsto_P$ 
	transformations. This 
	means that each $b_j \in W_i^+$, each $c_r \in W_i^-$, and, by inductive
	hypothesis, they admit justifications free of negative cycles. We can construct
	a justification ${\cal J}(a)$ for $a$, which is free of negative cycles and is 
	the union of all the justifications
	free of negative cycles of $b_1,\dots,b_k,c_1,\dots,c_h$ and the
	edges $(a^+,b_1^+,+), \dots, (a^+,b_k^+,+), (a^+,c_1^-,-), \dots, (a^+,c_h^-,-)$. Note that,
	with the exception of $a$,  
	the atoms involved in the justification ${\cal J}(a)$ are only 
	atoms of $W_i^+\cup W^-_i$.

\item 
	Let us now consider $a\in W_{i+1}^-$. 
	Again, we assume that $a \not\in W_i^-$.
	This means that in $P_{i+1}$ there are
	no rules left for $a$. Let us consider each individual rule for $a$ in $P$, of the
	generic form
	\begin{equation}\label{generic}
	a \uffa b_1, \dots, b_k, not\:c_1, \dots, not\:c_h 
	\end{equation}
	We consider two cases:

\begin{itemize}

\item $P_i \mapsto_N P_{i+1}$ or $P_i \mapsto_F P_{i+1}$. 
	By our assumption about the sequence of transformations, 
	we can conclude that the transformation $\mapsto_L$ has not been 
     	applied in removing rules whose head is $a$.
	In other words, each rule (\ref{generic}) has been 
	removed by either a $\mapsto_N$ or a $\mapsto_F$ transformation. 
	This implies that for each rule (\ref{generic}), there exists either 
	a $c_j \in W_i^+$ or a $b_l \in W_i^-$, i.e., 
	there exists $C^+ \subseteq W_i^+$ and $C^- \subseteq W_i^-$
	such that for each rule $r$ with $head(r)=a$, 
	$C^+ \cap neg(r) \ne \emptyset$ 
	or $C^- \cap pos(r) \ne \emptyset$. 
	Without loss of generality, we can assume that 
	$C^+$ and $C^-$ are minimal (w.r.t. $\subseteq$). 
	By inductive hypothesis, 
	we know that each element in $C^+$ and $C^-$ 
	posses a justification free of negative cycles which 
	contain only atoms in $WP_i$. 
	Similar to the first item, we have that 
	${\cal J}(a) = \bigcup_{c \in C^+ \cup C^-} {\cal J}(c) 
		\cup \{(a^-,c^+,-) \mid c \in C^+\} \cup
			\{(a^-,c^-,+) \mid c \in C^-\}$
	is a justification free of negative cycles for $a$ which,
	with the exception of $a^-$,  
	contains only atoms in $WP_i$.

\item $P_i \mapsto_L P_{i+1}$. 
	The fact that $a \in W_{i+1}^- \setminus W_i^-$ indicates that 
	all rules with $a$ as head have been removed. 
	In this case, there might be some rules with $a$ as its  
	head that have been removed by other transformations. Let 
	$R_1(a)$ (resp. $R_2(a)$) be the set of rules, whose head is $a$,
	which are removed by a transformation $\mapsto_F$ or $\mapsto_N$
	(resp. $\mapsto_L$). 
	
	Let $S$ be the set of atoms employed for the $\mapsto_L$ step 
	(i.e., the $i$-th step). Let $a_1,\ldots,a_s$ be an enumeration 
	of $S$. For a subset $X$ of $S$, 
	let $\min(X)$ denote the element in $X$ with the smallest index
	according to the above enumeration.
	
	Let
	\[
	\begin{array}{lcr}
	G_0 & = & \{ (a^-, b^-, +) \:|\: a\uffa body \in P_i, b = \min(body\cap S) \}\\
	G_{j+1} & = & \{ (b^-,c^-,+) \:|\: \exists (d^-,b^-,+) \in G_j, (b\uffa body)\in P_i, \\
		&   &			c = \min(body \cap S) \}
	\end{array}
	\]

	Because of the finiteness of $S$, there exists some $j$ such that 
	$G_j \subseteq \bigcup_{0 \le i \le j-1} G_i$. Let 	
	be the graph\footnote{
	Again, we define the graph by its set of edges. 
	} $G = \bigcup_{j\geq 0} G_j$. Because of the property of $S$,
	it is easy to see that for each atom $c$ in the graph $G$, 
	$support(c,G)$ is a LCE of $c$ w.r.t. $WP_i$ and $\emptyset$ 
	(w.r.t. the program $P_i$). Thus, 
	we have that $G$ is an off-line justification for $a$ in $P_i$.
	Furthermore, it contains only positive cycles and it
	is composed of atoms from $S\cup \{a\}$. 

	The construction of $G$ takes care of rules of the form (\ref{generic}), 
	which belong to $R_2(a)$. Similar to the previous case, we know that 
	for each atom $b$ such that $b^-$ is a node in $G$, 
	there exists $C_b^+ \subseteq W_{i-1}^+$ and $C_b^- \subseteq W_{i-1}^-$
	such that for each rule $r$ with $head(r) = b$ in $R_1(b)$, 
	$C_b^+ \cap neg(r) \ne \emptyset$ 
	or $C_b^- \cap pos(r) \ne \emptyset$. 
	$G$ can be extended to an off-line justification of $a$ by adding to its 
	the justifications of other atoms that falsify the rules in $R_1(b)$ 
	for every $b \in S$. 
	More precisely, for each atom $b$ such that $b^-$ is a node in $G$, 
	let 
	\[
		G_b = \bigcup_{c \in C_b^+ \cup C_b^-} {\cal J}(c) 
		\cup \{(b^-,c^+,-) \mid c \in C_b^+\} \cup
			\{(b^-,c^-,+) \mid c \in C_b^-\}.
	\]
	Note that ${\cal J}(c)$ in the above equation exists due to the 
	inductive hypothesis. Furthermore, each $G_b$ contains only 
	atoms in $WP_i$ with the exception of $b$ and therefore cannot 
	contain negative cycles. Thus, 
	$G' = G \cup \bigcup_{b \textnormal{ is a node in } G}  G_b$ 
	does not contain negative cycles. 
	It is easy to check that $support(c,G')$ 
	is a LCE of $c$ in $P$ w.r.t. $WP_{i+1}$ and $\emptyset$. 
	Thus, $G'$ is an off-line justification for $a$ 
	 in $P$	w.r.t. $WP_{i+1}$ and $\emptyset$. 

\end{itemize}
\end{itemize}
\hfill$\Box$

\newpage

\subsection*{Proof of Proposition \ref{propimp}.}

\noindent
{\em Proposition \ref{propimp}.}
Let $P$ be a program and $M$ an answer set. For each atom $a$,
there is  an off-line 
justification w.r.t. $M$ and ${\cal TA}_P(M)$
which does not contain negative cycles.

\medskip \noindent
{\bf Proof:}
The result is trivial, since all the elements in ${\cal TA}_P(M)$
are immediately set to false, and $NR(P,{\cal TA}_P(M))$ has a well-founded
model equal to $M$ (and thus all elements have justifications
free of negative cycles, from Lemma~\ref{good}).
\hfill$\Box$

\newpage

\subsection*{Proof of Proposition \ref{gammadelta}.}

The proof of this proposition will develop through a number of intermediate steps.
Let us start by introducing some notation. Given a program $P$ and given the
Herbrand universe $\cal A$, let
$nohead(P)=\{a \in {\cal A}\: :\:  \forall r\in P.\: a \neq head(r)\}$.
Furthermore, for two sets of atoms $\Gamma, \Delta$ such 
that $\Gamma \cap \Delta = \emptyset$, we define a 
program transformation $\rightarrow_{\langle \Gamma, \Delta \rangle}$ as follows.
The program $P'$, obtained from $P$ by 
\begin{itemize}

\item removing $r$ from $P$ if $pos(r) \cap \Delta \ne \emptyset$ or 
	$neg(r) \cap \Gamma \ne \emptyset$
	(remove rules that are inapplicable w.r.t. $\langle \Gamma, \Delta \rangle$).

\item replacing each remaining rule $r$ with $r'$ 
	where $head(r') = head(r)$, $pos(r') = pos(r) \setminus \Gamma$,
	and $neg(r') = neg(r) \setminus \Delta$ (normalize the body of the rules w.r.t. 
		$\langle \Gamma, \Delta \rangle$)
\end{itemize}
is said to be the result of the transformation $\rightarrow_{\langle \Gamma, \Delta \rangle}$.
We write $P \rightarrow_{\langle \Gamma, \Delta \rangle} P'$ to denote this fact. 

\medskip \noindent The following can be proven.

\begin{lemma} \label{l1}
Let $P$ be a program $\Gamma$ and $\Delta$ be 
two sets of atoms such that $\Gamma \subseteq facts(P)$,
$\Delta = \bigcup_{i=1}^k S_i \cup X$ where 
$X \subseteq nohead(P)$ and $S_1,\ldots,S_k$ is a sequence of sets 
of atoms such that $S_i \in cycles(\langle \emptyset,\emptyset\rangle)$ for $1 \le i \le k$. 
It holds that if $P \rightarrow_{\langle \Gamma , \Delta \rangle }P'$ then 
there exists a sequence of basic transformations 
$P \mapsto_{t_1} P_1 \mapsto_{t_2} \ldots \mapsto_{t_m} P'$ 
where $t_i \in \{P,N,S,F,L\}$ (see
the proof of Lemma~\ref{good} for the definition of these transformations).
\end{lemma}

\noindent {\bf Proof.}
We prove this lemma by describing the sequence of transformations $\mapsto$. 
Let $\Omega = \bigcup_{i=1}^k S_i$. 
The proof is based on the following observations:
\begin{enumerate}
\item Since $\Gamma$ is a set of facts, we can repeatedly apply the $\mapsto_N$ and $\mapsto_S$ 
transformations to $P$. The result is a program $P_1$ with 
the following properties: for every $r \in P_1$, 
there exists some $r' \in P$ with $neg(r') \cap \Gamma = \emptyset$ 

\begin{enumerate}
\item $neg(r) = neg(r')$

\item $head(r) = head(r')$ and 

\item $pos(r) = pos(r') \setminus \Gamma$. 
\end{enumerate}

\item Since $X$ is a set of atoms with no rules in $P_1$, 
we can repeatedly apply the $\mapsto_P$ and $\mapsto_F$ transformations 
to $P_1$ for the atoms belonging to $X$. 
The result is a program $P_2$ with the following properties: for every $r \in P_2$, 
there exists some $r' \in P_1$ with $pos(r') \cap X = \emptyset$ and 

\begin{enumerate}
\item $pos(r) = pos(r')$

\item $head(r) = head(r')$ and 

\item $neg(r) = neg(r') \setminus X$. 
\end{enumerate}

\item Since $\Omega$ is a set of atoms with cycles, we can apply the 
loop detection transformation $\mapsto_L$ 
for each of the loops in $\Omega$ to $P_2$; thus,  we obtain 
$P_3 = P_2 \setminus \{r \in P_2 \mid head(r) \in \Omega\}$.

\item Since atoms in $\Omega$ will no longer have defining rules in $P_3$, the transformations 
for atoms in $\Omega$ (similar to those for atoms in $X$) can be applied to $P_3$;
the result is the program $P_4$ with the property: for every $r \in P_4$, 
there exists some $r' \in P_3$ with $pos(r') \cap \Omega = \emptyset$ and 

\begin{enumerate}
\item $pos(r) = pos(r')$

\item $head(r) = head(r')$ and 

\item $neg(r) = neg(r') \setminus \Omega$. 
\end{enumerate}
\end{enumerate}

\noindent 
Finally, let us consider $P_4$;
for each rule $r \in P_4$, there is a rule $r' \in P$
such that $pos(r') \cap \Delta = \emptyset$,
$neg(r') \cap \Gamma = \emptyset$, and 

\begin{enumerate}
\item $pos(r) = pos(r') \setminus \Gamma$

\item $head(r) = head(r')$ and 

\item $neg(r) = neg(r') \setminus \Delta$. 
\end{enumerate}
This shows that $P \rightarrow_{\langle \Gamma, \Delta \rangle} P_4$.
\hfill$\Box$


\smallskip
\noindent 
For a program $P$, let $WF_P$ be its well-founded model. 
Let us define a sequence of programs $P_0, P_1,\ldots,P_k,...$ as follows:

\[
\begin{array}{lcl}
P_0 & = & P  \\
P_0 &  \rightarrow_{\langle \Gamma^1(WF_P), \Delta^1(WF_P) \rangle } & P_1 \\
P_i &  \rightarrow_{\langle \Gamma^{i+1}(WF_P), \Delta^{i+1}(WF_P) \rangle } & P_{i+1} \\
\end{array}
\]

\begin{lemma} \label{l2new}
Given the previously defined  sequence of programs, the following properties hold:
\begin{enumerate}
\item 
For $i \ge 0$, $\Gamma^{i}(WF_P) \subseteq facts(P_i)$ 
and $\Delta^{i}(WF_P)\subseteq nohead(P_i)$. 

\item 
If $\Gamma^i(WF_P) = \Gamma^{i+1}(WF_P)$ 
then $\Gamma^{i+1}(WF_P) = facts(P_{i+1})$.
 
\item 
If $\Delta^i(WF_P) = \Delta^{i+1}(WF_P)$ then 
	$\Delta^{i+1}(WF_P) = nohead(P_{i+1})$. 
\end{enumerate}
\end{lemma}

\noindent {\bf Proof.}
\begin{enumerate}
\item The first property holds because of the construction of $P_i$
and the definitions of $\Gamma^{i}(WF_P)$ 
and $\Delta^{i}(WF_P)$.

\item Consider some $a \in facts(P_{i+1})$. 
By the definition of $P_{i+1}$, there exists some rule $r \in P_i$ 
such that 
	\begin{itemize}
	\item $head(r) = a$,
	\item  $pos(r) \cap \Delta^i(WF_P) = \emptyset$, 
	\item $neg(r) \cap \Gamma^i(WF_P) = \emptyset$, 
	\item $pos(r) \setminus \Gamma^i(WF_P) = \emptyset$, and
	\item $neg(r) \setminus \Delta^i(WF_P) = \emptyset$.
	\end{itemize}
This implies that $pos(r) \subseteq \Gamma^i(WF_P)$ and
$neg(r) \subseteq \Delta^i(WF_P)$, i.e., $a \in \Gamma^{i+1}(WF_P)$.
This proves the equality of the second item.

\item Consider some $a \in nohead(P_{i+1})$. This means that 
every rule of $P_i$ having $a$ in the head has been removed;
i.e., for every $r \in P_i$ with $head(r) = a$, we have that
\begin{itemize}
	\item $pos(r) \cap \Delta^i(WF_P)  \ne \emptyset$ or
	\item $neg(r) \cap \Gamma^i(WF_P) \ne \emptyset$.
\end{itemize}
This implies 
that $a \in \Delta^{i+1}(WF_P)$, which allows us to conclude the
third property.
\end{enumerate}
\hfill$\Box$


\begin{lemma} \label{l3}
Let $k$ be the first index such that $\Gamma^k(WF_P) = \Gamma^{k+1}(WF_P)$ and
$\Delta^k(WF_P) = \Delta^{k+1}(WF_P)$. Then, $P_{k+1}$ 
is irreducible w.r.t. the transformations $\mapsto_{NPSFL}$. 
\end{lemma}

\noindent {\bf Proof.}
This results follows from Lemma \ref{l2new}, since
 $\Gamma^{k+1}(WF_P) = facts(P_{k+1})$ 
and $\Delta^{k+1}(WF_P) = nohead(P_{k+1})$.
This means that $P_{k+1} \mapsto^*_{NPSF} P_{k+1}$. Furthermore,
$cycles(\langle \Gamma^{k+1}(WF_P), \Delta^{k+1}(WF_P)) = \emptyset$.
Hence,  $P_{k+1}$ is irreducible. 
\hfill$\Box$

\begin{lemma} \label{prop-wfs}
For a program $P$, 
$WF_P = \langle \Gamma(WF_P), \Delta(WF_P) \rangle$.
\end{lemma}

\noindent {\bf Proof.}
This results follows  from Lemmas \ref{l2new} and \ref{l3}.
\hfill$\Box$

\begin{lemma} \label{l4}
Given two p-interpretations $I \sqsubseteq J$, we have that 
$\Gamma(I) \subseteq \Gamma(J)$ and $\Delta(I) \subseteq \Delta(J)$.
\end{lemma}

\noindent {\bf Proof.}
We prove that $\Gamma^i(I) \subseteq \Gamma^i(J)$ and 
$\Delta^i(I) \subseteq \Delta^i(J)$ by induction on $i$.

\begin{enumerate}
\item {\bf Base:} $i=0$. This step is obvious, since $I \subseteq J$.

\item {\bf Step:} Let $I_i = \langle \Gamma^i(I), \Delta^i(I) \rangle$ 
	and $J_i = \langle \Gamma^i(J), \Delta^i(J) \rangle$.
	From the inductive hypothesis, we can conclude  that $I_i \sqsubseteq J_i$.
	This result, together with the fact that, for any rule $r$, 
	$I_i \models body(r)$ implies $J_i \models body(r)$, allows us to conclude that 
	$\Gamma^{i+1}(I) \subseteq \Gamma^{i+1}(J)$. Similarly, from the 
	fact that $cycles(I_i) \subseteq cycles(J_i)$ and the 
	inductive hypothesis, we can show that 
	$\Delta^{i+1}(I) \subseteq \Delta^{i+1}(J)$. 
\end{enumerate}
\hfill$\Box$

\begin{lemma} \label{prop-ans}
Given a program $P$ and an answer set $M$ of $P$,
$M = \langle \Gamma(M), \Delta(M)\rangle$.
\end{lemma}

\noindent {\bf Proof.}
Let us prove this lemma by contradiction. 
Let $J = \langle \Gamma(M), \Delta(M)\rangle$.
First, Lemma \ref{l4} and \ref{prop-wfs} imply that $WF_P \subseteq J$.
Since $M$ is an answer set of $P$, there exists some level mapping 
$\ell$ such that $M$ is a \emph{well-supported 
model} w.r.t. $\ell$ \cite{Fages94}, i.e., for each $a \in M^+$ there 
exists a rule $r_a$
satisfying the following conditions:
\begin{itemize}
\item $head(r_a) = a$, 
\item $r_a$ is supported by $M$ (i.e., $pos(r_a) \subseteq M^+$ and $neg(r_a) \subseteq M^-$), and 
\item $\ell(a) > \ell(b)$ for each $b \in pos(r_a)$. 
\end{itemize}
We have to consider the following cases:

\begin{itemize}
\item {\bf Case 1:} $M^+ \setminus J^+ \ne \emptyset$. 
Consider $a \in M^+ \setminus J^+$ such that 
$\ell(a) = \min \{\ell(b)\: \mid\:  b \in M^+ \setminus J^+\}$.
There exists a rule $r$ such that 
 $head(r) = a$, $r$ is supported by $M$, and 
 $\ell(a) > \ell(b)$ for each $b \in pos(r)$. 
The minimality of $\ell(a)$ implies that $pos(r) \subseteq J^+$. 
The fact that $a \not\in J^+$ implies that $neg(r) \setminus J^- \ne \emptyset$.
Consider some $c \in neg(r) \setminus J^-$. 
Clearly, $c \not\in (NANT(P)  \setminus WF_P^-)$---otherwise, it would belong 
to $J^-$. This implies that  $c \in WF_P^-$ because $c \in NANT(P)$. 
Hence, $c \in J^-$. This represents a contradiction.

\item {\bf Case 2:} $M^- \setminus J^- \ne \emptyset$. 
Consider $a \in M^- \setminus J^-$. This is possible only 
if there exists some rule $r$ such that
	\begin{itemize}
	\item  $head(r) = a$, 
 	\item $pos(r) \cap \Delta(M) = \emptyset$,
	\item $neg(r) \cap \Gamma(M) = \emptyset$, and 
 	\item either 
		\begin{enumerate}
		\item[(i)]   $neg(r) \setminus \Delta(M) \ne \emptyset$, or 
		\item[(ii)] $pos(r) \setminus \Gamma(M) \ne \emptyset$.
		\end{enumerate}
	\end{itemize}
In what follows, by $R_a$ we denote the set of rules in $P$ whose 
head is $a$ and  whose bodies are neither true nor false in $J$. 

If (i) is true, then there exists some $b \in neg(r) \setminus \Delta(M)$.
Since $b \in neg(r)$, we have that $b \in NANT(P)$.
This implies that $b \not\in M^-$ or $b \in WF_P^-$. 
The second case cannot happen since $WF_P \sqsubseteq J$ (Lemma \ref{l4}).
So, we must have that $b \not\in M^-$. This means that $b \in M^+$ (since
$M$ is an answer set, and thus a complete interpretation), 
 and hence, $b \in J^+$ (Case 1). This contradicts the 
fact that $neg(r) \cap \Gamma(M) = \emptyset$. Therefore, we conclude
that (i) cannot happen. 

Since (i) is not true, we can conclude that $R_a \ne \emptyset$ 
and for every $r\in R_a$ and $b \in pos(r) \setminus \Gamma(M)$,
$b \in J^-\setminus M^-$ and $R_b \ne \emptyset$. Let us 
consider the following sequence:
\[
\begin{array}{l}  
C_0 = \{a\} \;\;\;\;\; \\
C_1 = \bigcup_{r \in R_a} (pos(r) \setminus \Gamma(M)) \\
\ldots  \\
C_i = \bigcup_{b \in C_{i-1}} (\bigcup_{r \in R_b} (pos(r) \setminus \Gamma(M))) \\
\end{array}
\]
Let $C = \bigcup_{i=0}^\infty C_i$. It is easy to see that for each 
$c \in C$, it holds that $c \in M^- \setminus J^-$, 
$R_c \ne \emptyset$, and for each $r \in R_c$, $pos(r) \cap C \ne \emptyset$.
This means that $C \in cycles(J)$. This is a contradiction with 
 $C \subseteq (M^- \setminus J^-)$. 
\hfill$\Box$
\end{itemize}


\medskip
\noindent
{\em Proposition \ref{gammadelta}.}
For a program $P$, we have that: 
\begin{itemize}
\item $\Gamma$ and $\Delta$ maintains the consistency of $J$, i.e., 
if $J$ is an interpretation, then $\langle \Gamma(J), \Delta(J) \rangle$ 
is also an interpretation;
\item $\Gamma$ and $\Delta$ are monotone w.r.t the argument $J$, i.e.,
if $J \sqsubseteq J'$ then $\Gamma(J) \subseteq \Gamma(J')$ and $\Delta(J) \subseteq \Delta(J')$;
\item $\Gamma(WF_P) = WF_P^+ $ and $\Delta(WF_P) = WF_P^-$; and 
\item if $M$ is an answer set of $P$, then 
	$\Gamma(M) = M^+ $ and $\Delta(M) = M^-$.
\end{itemize}

\noindent
{\bf Proof:}
\begin{enumerate}

\item Follows immediately from the definition of $\Gamma$ and $\Delta$.

\item Since $J \models body(r)$ implies $J' \models body(r)$ 
and $S \in cycles(J)$ implies $S \in cycles(J')$ if $J \subseteq J'$,
the conclusion of the item is trivial.

\item This derives from  Lemma \ref{prop-wfs}.
 
\item This derives from Lemma  \ref{prop-ans}.
 
\end{enumerate}
\hfill $\Box$

\newpage
\subsection*{Proof of Proposition \ref{on-off}.}

To prove this proposition, we first prove Lemma \ref{just-free} and \ref{conserve}. 
We need the following definition. 

\begin{definition}[Subgraph]
Let $G$ be an arbitrary graph whose nodes are in 
${\cal A}^p \cup {\cal A}^n \cup \{assume,\top,\bot\}$ and whose
edges are labeled with $+$ and $-$. 

Given $e \in  {\cal A}^+ \cup {\cal A}^-$, the
\emph{subgraph} of $G$ with root $e$, 
denoted by  $Sub(e, G)$, is the graph obtained from $G$ by 
\begin{enumerate}
	\item[(i)] removing all the edges of $G$  which 
	do not lie on any path starting from $e$, and 
	\item[(ii)] removing all nodes unreachable from $e$ in the resulting graph.
\end{enumerate}
\end{definition}

Throughout this section, let $I_i$ denote $\langle \Gamma^i(J), \Delta^i(J)\rangle$. 
For a set of atoms $C$ and an element $b \in C$, let 
\begin{equation} \label{kbc}
K(b,C) = \{c \mid c \in C, \exists r\in P. \: (head(r)=b, c\in pos(r))\}.
\end{equation}

\begin{lemma} \label{split-d0}
For a p-interpretation $J$ and $A = {\cal TA}_P(J)$,
let $\Delta^0(J) = \Delta_1^0 \cup \Delta_2^0 \cup \Delta_3^0$ 
where $$\Delta_1^0 = \{a \in \Delta^0(J) \mid PE(a^-,\langle \emptyset,\emptyset\rangle) \ne \emptyset\},$$
$$\Delta_2^0 = \{a \in \Delta^0(J) \mid  
	PE(a^-,\langle \emptyset,\emptyset\rangle) = \emptyset \textnormal{ and } a \in {\cal TA}_P(J)\},$$
and 
$$\Delta_3^0 = \{a \in \Delta^0(J) \mid  
	PE(a^-,\langle \emptyset,\emptyset\rangle) = \emptyset \textnormal{ and } a \not\in {\cal TA}_P(J)\}.$$
The following properties hold:
\begin{itemize}
\item $\Delta^0_1$, $\Delta^0_2$, and $\Delta^0_3$ are pairwise disjoint. 
\item for each $a \in \Delta_3^0$ there exists a LCE 
	  $K_a$ of $a^-$ w.r.t. $\langle \Gamma^0(J),\Delta^0(J)\rangle$ 
	and $A$ such that for each rule $r \in P$ with $head(r) = a$, $pos(r) \cap K_a \ne \emptyset$. 
\end{itemize}
\end{lemma}

\noindent {\bf Proof.}
The first item is trivial thanks to the definition of $\Delta^0_1$, $\Delta^0_2$, and $\Delta^0_3$.
For the second item, for $a \in \Delta_3^0$, there exists some 
$C \in cycles(\langle \emptyset,\emptyset \rangle)$ such that $a \in C$ 
and $C \subseteq \Delta^0(J)$. From the definition of a cycle, there exists some $K_a \subseteq C \subseteq \Delta^0(J)$ which 
satisfies the condition of the second item.
\hfill$\Box$

\smallskip
\noindent 
We will now proceed to prove Lemma \ref{just-free}.  
For each $i$, we construct a dependency graph $\Sigma_i$ 
for elements in $(\Gamma_i(J))^p$ and $(\Delta_i(J))^n$ as follows.
Again, we describe a graph by its set of edges. 
First, the construction of $\Sigma_0$ is as follows. 
\begin{enumerate}

\item for each $a \in \Gamma^0(J)$, $\Sigma_0$ contains the edge $(a^+,\top,+)$.

\item let $\Delta^0_1$, $\Delta^0_2$, and $\Delta^0_3$ be defined as in Lemma~\ref{split-d0}.
	\begin{enumerate}
	\item For $a \in \Delta^0_1$, $\Sigma_0$ contains the edge $(a^-,\bot,-)$;
	\item For $a \in \Delta^0_2$, $\Sigma_0$ contains the edge $(a^-,assume,-)$. 
	\item Let $a \in \Delta^0_3$. This implies that 
      		there exists some $C  \subseteq J^-$ 
		such that $a \in C$ and $C \in cycles(\langle \emptyset,\emptyset\rangle)$. 
	For each $b \in C$, let $K_b$ be an explanation of $b^-$ w.r.t. $\langle \Gamma^0(J),\Delta^0(J)\rangle$ 
	and ${\cal TA}_P(J)$ which satisfies the conditions of the second item in Lemma
	\ref{split-d0}. 
	Then, 	$\Sigma_0$ contains the set of edges
	$\bigcup_{b \in C} \{(b^-,c^-,+)\: \mid\: c \in K_b\})$. 
	
\end{enumerate}

\item no other edges are added to $\Sigma_0$.
\end{enumerate}

\begin{lemma} \label{lbase}
Let $J$ be a p-interpretation and $A = {\cal TA}_P(J)$. 
The following holds for $\Sigma_0$:
\begin{enumerate}
	\item for each $a \in  \Gamma^0(J)$, $Sub(a^+, \Sigma_0)$ is a safe off-line e-graph of $a^+$ 
		w.r.t. $I_0$ and $A$. 
	\item for each $a \in  \Delta^0(J)$, $Sub(a^-, \Sigma_0)$ is a safe off-line e-graph of $a^-$ 
		w.r.t. $I_0$ and $A$. 
\end{enumerate}
\end{lemma}

\noindent {\bf Proof.}

\begin{itemize}

\item Consider $a \in \Gamma^0(J)$. Since $\Sigma_0$ contains 
	$(a^+,\top,+)$ for every $a \in \Gamma^0(J)$ and $\top$ is a sink 
	in $\Sigma_0$, we can conclude that 
	$Sub(a^+,\Sigma_0) = (\{a^+,\top\},\{(a^+,\top,+)\})$ 
	and $Sub(a^+,\Sigma_0)$
	is a safe off-line e-graph of $a^+$ w.r.t. $I_0$ and $A$. 

\item Consider $a \in \Delta^0(J)$. 
Let $\Delta^0_1$, $\Delta^0_2$, and $\Delta^0_3$ be defined as in Lemma \ref{split-d0}.
There are three cases:

\begin{enumerate}

\item $a \in \Delta^0_1$.
	Since $\Sigma_0$ contains $(a^-,\bot,-)$ and $\bot$ is a sink 
	in $\Sigma_0$, we can conclude that 
	$Sub(a^-,\Sigma_0) = (\{a^-,\bot\},\{(a^-,\bot,-)\})$ 
	and $Sub(a^-,\Sigma_0)$
	is a safe off-line e-graph of $a^-$ w.r.t. $I_0$ and $A$. 

\item $a \in \Delta^0_2$.
	Since $\Sigma_0$ contains $(a^-,assume,-)$ and $assume$ is a sink 
	in $\Sigma_0$, we can conclude that 
	$Sub(a^-,\Sigma_0) = (\{a^-,assume\},\{(a^-,assume,-)\})$ 
	and $Sub(a^-,\Sigma_0)$
	is a safe off-line e-graph of $a^-$ w.r.t. $I_0$ and $A$. 

\item for $a \in \Delta^0_3$, let $G = Sub(a^-,\Sigma_0)  = (N,E)$. 
	It is easy to see that $G$ is 
	indeed a $(J,A)$-based e-graph of $a^-$ because, from the construction 
	of $G$, we have that
		\begin{enumerate}
		\item[(i)] every node in $N$ is reachable from $a^-$,	and 
		\item[(ii)] if $b^- \in N$ then $support(b^-,G) = K_b \subseteq N$
		is a local consistent  explanation of $b^-$ w.r.t. $I_0$ and $A$.
		\end{enumerate}
	The safety of the e-graph derives from the fact 
	that it does not contain any nodes of the form $p^+$. 
\end{enumerate}
\hfill$\Box$
\end{itemize}

\noindent To continue our construction, we will need the following lemma.

\begin{lemma} \label{split-di}
Let $J$ be a p-interpretation and $A = {\cal TA}_P(J)$ and $i > 0$. 
Let 
$$
\Delta^i_1 = \{a \in \Delta^i(J) \setminus \Delta^{i-1}(J) \mid 
			PE(a^-,I_{i-1}) \ne \emptyset\} 
$$
and 
$$
\Delta^i_2 = \{a \in \Delta^i(J) \setminus \Delta^{i-1}(J) \mid 
			PE(a^-,I_{i-1}) = \emptyset\}. 
$$
Then, 
\begin{itemize}
\item $\Delta^i_1 \cap \Delta^i_2 = \emptyset$; 

\item 
for each $a \in \Delta^i_1$ there exists some LCE
$K_a$ of $a^-$ w.r.t. $I_i$ and $A$ such that 
$\{p \in {\cal A} \mid p \in K_a\} \subseteq \Delta^{i-1}(J)$ and 
$\{a \mid \naf a \in K_a\} \subseteq \Gamma^{i-1}(J)$; and 

\item 
for each $a \in \Delta^i_2$ there exists some LCE
$K_a$ of $a^-$ w.r.t. $I_i$ and $A$ such that 
$\{p \in {\cal A} \mid p \in K_a\} \subseteq \Delta^{i}(J)$ and 
$\{a \mid \naf a \in K_a\} \subseteq \Gamma^{i-1}(J)$.
\end{itemize}
\end{lemma}

\noindent {\bf Proof.}
These properties follow immediately from the definition of $\Delta^i(J)$.
\hfill$\Box$

\bigskip
\noindent
Given  $\Sigma_{i-1}$, we can construct $\Sigma_i$ by reusing
all nodes and edges of  $\Sigma_{i-1}$ along with  the following nodes and edges. 

\begin{enumerate}
\item for each $a \in \Gamma^i(J) \setminus \Gamma^{i-1}(J)$, from the 
definition of $\Gamma^i(J)$ we know 
that there exists a rule $r$ such that $head(r) = a$, 
$pos(r) \subseteq \Gamma^{i-1}(J)$,  and 
$neg(r) \subseteq \Delta^{i-1}(J)$.  
$\Sigma_i$ contains the node $a^+$ and the set of edge
$\{(a^+,b^+,+) \mid b \in pos(r)\} \cup \{(a^+,b^-,+) \mid b \in neg(r)\}$.

\item let $\Delta^i_1$ and $\Delta^i_2$ be defined as in Lemma
\ref{split-di}. 

\begin{enumerate}

\item For $a \in \Delta^i_1$, let $K_a$ be a LCE 
of $a^-$ satisfying the second condition of Lemma \ref{split-di}. 
	Then, $\Sigma_i$ contains the following set of edges:
	$\{(a^-,b^-,+) \mid b \in K_a\} \cup \{(a^-,b^+,-) \mid \naf b \in K_a\}$;

\item For $a \in \Delta^i_2$, let $K_a$ be a LCE 
of $a^-$ satisfying the third condition of Lemma \ref{split-di}. 
Then, 	$\Sigma_i$ contains the set of links
$$
\{(a^-,c^-,+) \mid c \in K_b\} 
\cup \{(a^-,c^+,-) \mid \naf c \in K_b\}. 	
$$

\end{enumerate}

\item no other links are added to $\Sigma_i$.
\end{enumerate}

\begin{lemma} \label{lprop-sigma}
Let $J$ be p-interpretation and $A = {\cal TA}_P(J)$. 
For every integer $i$, the following properties hold:
\begin{enumerate}
\item for each $a \in  \Gamma^i(J) \setminus \Gamma^{i-1}(J)$,
$Sub(a^+, \Sigma_i)$ is a safe off-line e-graph of $a^+$ 
w.r.t. $I_i$ and $A$. 
\item for each $a \in  \Delta^i(J) \setminus \Delta^{i-1}(J)$,
$Sub(a^-, \Sigma_i)$ is a safe off-line e-graph of $a^-$ 
w.r.t. $I_i$ and $A$. 
\end{enumerate}
\end{lemma}

\noindent
{\bf Proof.}
The proof is done by induction on $i$. The base case is proved in 
Lemma \ref{lbase}. Assume that we have proved the lemma for $j < i$.
We now prove the lemma for $i$. We consider two cases:

\begin{enumerate}

\item $a \in \Gamma^i(J) \setminus \Gamma^{i-1}(J)$. Let $r$
be the rule 
with $head(r) = a$ used in Item 1 of the construction of $\Sigma_i$. 
For each $b \in pos(r)$, let $P_b = (NP_b,EP_b) = Sub(b^+,\Sigma_{i-1})$.
For each $b \in neg(r)$, let $Q_b = (NQ_b,EQ_b) = Sub(b^-,\Sigma_{i-1})$.
We have that 
$Sub(a^+,\Sigma_i) = (N,G)$ where
\[
\begin{array}{lll}
N & = & \{a^+\}\ \cup  \{b^+ \mid b \in pos(r)\} \cup \{b^- \mid b \in neg(r)\} \cup \\
  &   &   \bigcup_{b \in pos(r)} NP_b \cup \bigcup_{b \in neg(r)} NQ_b
\end{array}
\]
and 
\[
\begin{array}{lll}
E & = & \{(a^+,b^+,+) \mid b \in pos(r)\} \cup 
   \{(a^+,b^-,+) \mid b \in neg(r)\} \cup \\ 
  & & \bigcup_{b \in pos(r)} EP_b \cup \bigcup_{b \in neg(r)} EQ_b
\end{array}
\]
   From the inductive hypothesis, we have that 
$P_b$'s (resp. $Q_b$'s) are safe off-line e-graphs of $b^+$ (resp. $b^-$) 
w.r.t. $I_{i-1}$ and $A$. This implies that 
$G$ is a $(I_i,A)$-based e-graph of $a^+$. 
Furthermore, for every $(a^+,e,+) \in E$,
$e \in (\Gamma^{i-1}(J))^p$ or $e \in (\Delta^{i-1}(J))^n$. 
Thus, $(a^+,a^+) \notin E^{*,+}$. 

\item for each $a \in \Delta^i(J) \setminus \Delta^{i-1}(J)$, 
let $G = Sub(a^-,\Sigma_i) = (N,E)$. From the definition of
$G$, every node in $N$ is reachable from $a^-$ and 
$support(e,G)$ is a local consistent explanation of $e$ w.r.t. $I_i$ and $A$ 
for every $e \in N$. Thus, $G$ is a $(I_i,A)$-based e-graph of 
$a^-$. Furthermore, it follows from the definition of $\Sigma_i$ that 
there exists no node $e \in N$ such that $e \in (\Gamma^i(J) \setminus \Gamma^{i-1}(J))^+$.
Thus, if $c^+ \in N$ and $(c^+,c^+) \in E^{*,+}$ then 
we have that $Sub(c^+,\Sigma_{i-1})$ is not safe. This contradicts the fact
that it is safe due to the inductive hypothesis. 
\end{enumerate}
\hfill$\Box$

\medskip 
\noindent
{\em Lemma \ref{just-free}.}
Let $P$ be a program, $J$ a p-interpretation, and 
$A = {\cal TA}_P(J)$. The following properties hold:
\begin{itemize}
\item For each atom $a \in \Gamma(J)$ (resp. $a \in \Delta(J)$), 
there exists a \emph{safe} off-line e-graph of $a^+$ (resp. $a^-$) 
w.r.t. $J$ and $A$; 
\item for each atom $a \in J^+ \setminus \Gamma(J)$
(resp. $a \in J^- \setminus \Delta(J)$) there exists an on-line 
e-graph of $a^+$ (resp. $a^-$) w.r.t. $J$ and $A$. 
\end{itemize}

\noindent
{\bf Proof.}
The first item follows from the Lemma \ref{lprop-sigma}. 
The second item of the lemma is trivial due to the fact 
that
$(\{a^+,assume\}, \{(a^+,assume,+)\})$  \\
(resp. $(\{a^-,assume\}, \{(a^-,assume,-)\})$) is a
$(J,A)$-based e-graph of $a^+$ (resp. $a^-$), and hence,
is an on-line e-graph of $a^+$ (resp. $a^-$) 
w.r.t. $J$ and $A$. 
\hfill$\Box$

\bigskip \noindent
{\em Lemma \ref{conserve}.} 
Let $P$ be a program, $J$ be an interpretation, and $M$ be an answer set such 
that $J \sqsubseteq M$. For every atom $a$, if $(N,E)$ is a safe off-line e-graph 
of $a$ w.r.t. $J$ and $A$ where $A = J^- \cap {\cal TA}_P(M)$ then 
it is an off-line justification of $a$ w.r.t. $M$ and ${\cal TA}_P(M)$.

\medskip
\noindent
{\bf Proof.} The result is obvious from the definition of off-line 
e-graph and from the fact that $J^- \cap {\cal TA}_P(M) \subseteq {\cal TA}_P(M)$.
\hfill $\Box$

\bigskip \noindent 
{\em Proposition \ref{on-off}.} 
Let $M_0, \ldots, M_k$ be a general complete computation and  \\
$S(M_0), \ldots, S(M_k)$ be an on-line justification of the computation.
Then, for each atom $a$ in $M_k$, the  e-graph of $a$ in $S(M_k)$ is 
an off-line justification of $a$ w.r.t. $M_k$ and ${\cal TA}_P(M_k)$.

\medskip
\noindent
{\bf Proof.} This follows immediately from Lemma \ref{conserve},
the construction of the snapshots $S(M_i)$, 
the fact that $M_i \sqsubseteq M_k$ for every $k$, and $M_k$ is 
an answer set of $P$.
\hfill $\Box$

\end{document}